\documentclass[11pt]{article}
\usepackage{times}
\usepackage{amsmath,amssymb,amsthm}
\usepackage{bbm}

\usepackage{amsmath,amsfonts,color}
\usepackage{tcolorbox,amssymb}
\usepackage{hyperref}
\usepackage{cleveref}
\usepackage{natbib}
\usepackage{framed}

\usepackage{algorithm}
\usepackage{mathtools}
\usepackage{enumitem}

\usepackage{fullpage}

\newcommand{\cH}{\mathcal{H}}
\newcommand{\cC}{\mathcal{C}}
\newcommand{\ldim}[1]{\mathrm{Ldim}(#1)}
\newcommand{\ignore}[1]{}

\newcommand{\N}{{\mathbb N}}
\newcommand{\1}{{\mathbbm 1}}

\DeclareMathOperator*{\poly}{\mathrm{poly}}

\newcommand{\R}{{\mathbb R}}

\newcommand{\eps}{\epsilon}

\newcommand{\anote}[1]{{\color{magenta} Amos: #1}}

\newcommand{\new}[1]{{\color{red} #1}}

\newcommand{\calB}{{\cal B}}

\newcommand{\db}{S}

\newcommand{\stepref}[1]{Step~(\ref{#1})}
\newcommand{\steprefrange}[2]{Steps~(\ref{#1})--(\ref{#2})}

\newcommand{\error}{{\rm error}}

\newcommand{\DDD}{\mathcal P}

\newcommand{\maj}{\mathsf{MAJ}}

\newcommand{\test}{\operatorname{\rm test}}
\newcommand{\testExample}{\operatorname{\rm test}}
\newcommand{\inr}{\in_{\mbox{\tiny R}}}

\newcommand{\remove}[1]{}
\newcommand{\AAA}{\mathcal A}
\newcommand{\BBB}{\mathcal B}
\newcommand{\Alg}{\AAA}

\newcommand{\AlgLearnComposition}{\AAA_{\rm ClosureLearn}} 
\newcommand{\AlgClosureLearn}{\AAA_{\rm ClosureLearn}}
\newcommand{\AlgPrivateAgnostic}{\AAA_{\rm PrivateAgnostic}}

\newcommand{\AlgRelabel}{\AAA_{\rm Relabel}}
\newcommand{\AlgRelabelAndLearn}{\AAA_{\rm RelabelAndLearn}}

\newcommand{\ppp}{\mathsf p}

\DeclarePairedDelimiter{\ourset}{\{}{\}}
\newcommand{\VC}{\operatorname{\rm VC}}
\renewcommand{\epsilon}{\varepsilon}

\newcommand{\E}{\operatorname*{\mathbb{E}}}

   {\endMakeFramed}

\newtheorem{theorem}{Theorem}[section]
\newtheorem{lemma}[theorem]{Lemma}
\newtheorem{claim}[theorem]{Claim}

\newtheorem{definition}[theorem]{Definition}

\newtheorem{example}[theorem]{Example}

\newtheorem{prop}[theorem]{Proposition}
\newtheorem{obs}[theorem]{Observation}

\newtheorem{question}[theorem]{Question}

\theoremstyle{remark}
\newtheorem{remark}[theorem]{Remark}

\Crefname{equation}{}{}
\Crefname{lemma}{Lemma}{Lemmas}
\renewcommand{\cref}{\Cref}

\makeatletter
\def\@opargbegintheorem#1#2#3{\trivlist
\item[\hskip\dimexpr\labelsep\relax{\bf #1\ #2}]\hspace{-3px}{\bf (#3)}\ \itshape}
\makeatother

\title{Closure Properties for Private Classification and Online Prediction}
\author{
Noga Alon\thanks{
Department of Mathematics, Princeton University, Princeton, New Jersey, USA 
and Schools of Mathematics and Computer Science, Tel Aviv University, Tel Aviv, Israel.
Research supported in part by NSF grant DMS-1855464, ISF grant 281/17, BSF grant 2018267 and the Simons Foundation.}
\and
Amos Beimel\thanks{Department of Computer Science, Ben-Gurion University, Beer-Sheva, Israel.
This work was done while  visiting Georgetown University, supported by NSF grant no.~1565387, TWC: Large: Collaborative: Computing Over Distributed Sensitive Data and by ERC grant 742754 (project NTSC) and also supported by ISF grant 152/17 and  by a grant from the Cyber Security Research Center at Ben-Gurion University of the Negev.}
\and
Shay Moran\thanks{Google AI, Princeton.}
\and
Uri Stemmer\thanks{Department of Computer Science, Ben-Gurion University, Beer-Sheva, Israel, and Google Research. Partially supported by ISF grant 1871/19.}
}

\begin{document}
\maketitle
\begin{abstract}
Let~$\cH$ be a class of boolean functions and consider a {\it composed class} $\cH'$ that is derived from~$\cH$ using some arbitrary aggregation rule (for example, $\cH'$ may be the class of all 3-wise majority-votes of functions in $\cH$). We upper bound the Littlestone dimension of~$\cH'$ in terms of that of~$\cH$. 
As a corollary, we derive closure properties for online learning and private PAC learning.

The derived bounds on the Littlestone dimension exhibit an undesirable exponential dependence. For private learning, we  prove close to optimal bounds that circumvents this suboptimal dependency. The improved bounds on the sample complexity of private learning are derived algorithmically via transforming a private learner for the original class $\cH$ to a private learner for the composed class~$\cH'$. Using the same ideas we show that any ({\em proper or improper}) private algorithm that learns a class of functions $\cH$ in the realizable case (i.e., when the examples are labeled by some function in the class) can be transformed to a private algorithm that learns the class $\cH$ in the agnostic case.
\end{abstract}

\section{Introduction}

We study closure properties for learnability of binary-labeled hypothesis classes in two related settings:
online learning and differentially private PAC learning.

\paragraph{Closure Properties for Online Learning.}
Let $\cH$ be a class of experts that can be online learned with vanishing regret. 
	That is, there exists an algorithm $\Alg$ such that given any sequence of $T$ prediction tasks,
	the number of false predictions made by $\Alg$ is larger by at most $R(T)=o(T)$ 
	than the number of false predictions made by the best expert in $\cH$. 

Consider a scenario where the sequence of tasks is such that every single expert in $\cH$ predicts poorly on it, 
	however there is a small unknown set of experts $h_1,\ldots, h_k\in\cH$ that can predict well by collaborating.
	More formally, there is an aggregation rule $G:\{0,1\}^k\to\{0,1\}$ 
	such that the combined expert $G(h_1,\ldots, h_k)$ exhibits accurate predictions on a significant majority of the tasks. 
	For example, a possible aggregation rule $G$ could be the majority-vote of the $k$ experts.
	Since we assume that the identities of the $k$ experts are not known, it is natural to consider the class 
	$\cH'=\{G(h_1,\ldots,h_k) : h_i\in \cH\},$ 
	which consists of all possible $G$-aggregations of $k$ experts from $\cH$.
	We study the following question:

\begin{question}\label{q:online}
Can the optimal regret with respect to $\cH'$ be bounded in terms of that of~$\cH$?
\end{question}


The {\it Littlestone dimension} is a combinatorial parameter that determines online learnability~\citep{Littlestone87online,Bendavid09agnostic}.
	In particular, $\cH$~is online learnable if and only if it has a finite Littlestone dimension~$d<\infty$, and the best possible regret $R(T)$ for online learning~$\cH$ satisfies 
	\begin{equation}\label{eq:ldimonline}
	\Omega (\sqrt{dT}) \leq R(T) \leq O(\sqrt{dT\log T}).
	\end{equation}
Furthermore, if it is known that if one of the experts never errs (a.k.a the realizable setting), 
then the optimal regret is exactly $d$.\footnote{More precisely, there is a deterministic algorithm which makes no more than $d$ mistakes, and for every deterministic algorithm there is a (realizable) input sequence on which it makes at least $d$ mistakes. For randomized algorithms a slightly weaker lower bound of $d/2$ holds with respect to the expected number of mistakes.}
(The regret is called mistake-bound in this context.)

Thus, the above question boils down to asking whether the Littlestone dimension of $\cH'$ is bounded by a function of the Littlestone dimension of $\cH$.
One of the two main results in this work provides an affirmative answer to this question (\Cref{thm:littlestone}).

\vspace{2mm}

We next discuss a variant of this question in the setting of Differentially Private (DP) learning.
	The two settings of online and DP-learning are intimately related (see, e.g.,~\cite{Bun20Privateonline,Abernathy17onlilnedp,Joseph2019TheRO,Gonen19privateonline}).
	In particular, both online learning and DP-learning are characterized by the finiteness of 
	the Littlestone dimension~\citep{Littlestone87online,Bendavid09agnostic,BunNSV15,AlonLMM19,Bun20Privateonline}.

\paragraph{Closure Properties for Differentially Private Learning.}

Imagine the following medical scenario:
consider a family $\cH$ of viruses for which there is an algorithm $\Alg$ 
	that can learn to diagnose any specific virus~$h\in \cH$ given enough labeled medical data.
	Further assume that~$\Alg$ has the desired property of being differentially private learning algorithm as defined by~\citep{KasiviswanathanLNRS11};
	that is, it is a PAC learning algorithm in which the privacy of every patient whose data is used during training is guarded in the formal sense of differential privacy~\citep{DworkMNS06}.
	
Assume that an outbreak of a deadly disease $h'$ has occurred in several locations all over the world 
	and that it is known that $h'$ is caused by some relatively small, yet unknown group of viruses from $\cH$.
	That is, our prior information is that there are unknown viruses $h_1,\ldots, h_k\in \cH$ for a relatively small $k$ 
	such that $h'=G(h_1,\ldots, h_k)$ for some rule $G$.
	For example, $G$ could be the OR function in which case $h'$ occurs if and only if 
	the patient is infected with at least one of the viruses $h_1,\ldots, h_k$.

It would be highly beneficial if one could use the algorithm $\Alg$ to diagnose $h'$
	in an automated fashion. Moreover, doing it in a private manner could encourage health institutions
	in the different locations to contribute their patients' data. 
	This inspires the following question: 
\begin{question}\label{q:privacy}
Can one use the algorithm $\Alg$ to \emph{privately} learn to diagnose $h'$? 
How does the sample complexity of this learning task scale as a function of $G$?
\end{question}

\paragraph{Differential Privacy, Online Learning, and the Littlestone Dimension.}

Question~\ref{q:privacy} and Question~\ref{q:online} are equivalent in the sense that both 
online learning and DP-learning are characterized by the finiteness of 
the Littlestone dimension~\citep{Littlestone87online,Bendavid09agnostic,BunNSV15,AlonLMM19,Bun20Privateonline}.

Note however that unlike the bounds relating the Littlestone dimension to online learning,
	which are tight up to logarithmic factors (see \Cref{eq:ldimonline}),
	the bounds relating the Littlestone dimension and DP-learning
	are {\it very far from each other}; specifically, if $d$ denotes the Littlestone dimension of $\cH$
then the lower bound on the sample complexity of privately learning $\cH$ scales with $\log^* d$~\citep{BunNSV15,AlonLMM19}, 
while the best known\footnote{The lower bound is tight up to polynomial factors \citep{kaplan2019privately}, 
however the upper bound is not known to be tight: 
for example, as far as we know, it is possible that the sample complexity of private learning scales linearly
with $\VC(\cH) + \log^*(\ldim{\cH})$.} upper bound scales with $\exp(d)$~\citep{Bun20Privateonline}. 

Thus, while our solution to Question~\ref{q:online} yields an affirmative answer to Question~\ref{q:privacy}, 
the implied quantitative bounds are far from being realistically satisfying. 
Specifically, every finite $\cH$ is learnable with privacy using $O(\log |\cH|)$ samples~\citep{KasiviswanathanLNRS11}, 
and so if $\cH$ is finite and not too large, the bounds implied by the Littlestone dimension are not meaningful.
We therefore focus on deriving effective bounds for private learning,
which is the content of \Cref{thm:private} (see \Cref{thm:privacy} for a precise statement).

\paragraph{Littlestone Classes.}
It is natural to ask which natural hypothesis classes have bounded Littlestone dimension.
    First, it holds that $\ldim{\cH} \leq \log\lvert \cH\rvert$ for every $\cH$, 
    so for finite classes the Littlestone dimension scales rather gracefully with their size.

There are also natural infinite Littlestone classes:
    for example, let the domain $X=\mathbb{F}^n$ be an $n$-dimensional vector space over some field $\mathbb{F}$
    and let $\cH\subseteq\{0,1\}^X$ consist of all affine subspaces of $V$ of dimension $\leq d$.
    It can be shown here that $\ldim{\cH}=d$.
    (For example, the class of all lines in $\mathbb{R}^{100}$ has Littlestone dimension $1$.)
    A bit more generally, any class of hypotheses that can be described by
    polynomial \emph{equalities} of a bounded degree has bounded Littlestone dimension.
    (Observe that if one replaces ``equalities'' with ``inequalities''
    then the Littlestone dimension may become unbounded, however the VC dimension remains bounded
    (e.g.\ Halfspaces).)
    We note in passing that this can be further generalized to classes that are definable 
    in {\it stable theories}, which is a deep and well-explored notion in model theory.
    We refer the reader to~\cite{Chase19modelmachine}, Section 5.1 for such examples.

\paragraph{Organization.}

Formal statement of our main results and description of our techniques appears in \Cref{sec:main}, specifically,
a short overview of the proofs is given in \Cref{sec:technical}.
Definitions and background results are provided in \Cref{sec:Preliminaries}.
The complete proofs appear in the rest of the paper. 
Closure properties for Littlestone classes is proved in \Cref{sec:littlestoneproof}.
The effective bounds for  private learning are given in  \Cref{sec:relabel} and \Cref{sec:agnostic,sec:privacyproof}.
We note that each of these parts can be read independently of the other.


\section{Main Results and Techniques}
\label{sec:main}

Let $G:\{0,1 \}^k \to\{0,1\}$ be a boolean function and let $\cH_1,\ldots, \cH_k\subseteq \{0,1\}^X$ be hypothesis classes.
Denote by $G(\cH_1,\ldots,\cH_k)$ the following class
$ G(\cH_1,\ldots,\cH_k) = \{G(h_1,\ldots,h_k) : h_i\in \cH_i\}.$
For example, if $G(x_1,x_2) = x_1 \land x_2$ then 
$G(\cH_1,\cH_2) = \cH_1 \land \cH_2 = \{h_1\land h_2 : h_i\in\cH_i\}$ 
is the class of all pairwise intersections/conjunctions of a function from $\cH_1$ and a function from $\cH_2$.

\begin{theorem}[A Closure Theorem for the Littlestone Dimension]\label{thm:littlestone} 
Let $G:\{0,1\}^k\to\{0,1\}$ be a boolean function,
let~$\cH_1,\ldots,\cH_k\subseteq\{0,1\}^X$ be classes, and let $d\in\mathbb{N}$
such that~$\ldim{\cH_i} \leq d$ for every $i\leq k$. Then,
\[   \ldim{G(\cH_1,\ldots,\cH_k)} \leq \tilde O(2^{2k}k^2d),\]
where $\tilde O$ conceals polynomial factors in $\log k$ and $\log d$.
\end{theorem}
In particular, if $\ldim{\cH_i} < \infty$ for all $i\leq d$ then $\ldim{G(\cH_1,\ldots, \cH_k)}<\infty$.
    Consequently, if each of the $\cH_i$'s is online learnable then $G(\cH_1,\ldots, \cH_k)$ is online learnable. 
    We comment that if the aggregating function $G$ is simple then one can obtain better bounds.
    For example, if $G$ is a majority-vote, a $k$-wise OR, or a $k$-wise AND function then
    a bound of $\tilde O(k^2\cdot d)$ holds. (See \cref{sec:littlestonecomp}.)

Another combinatorial parameter which arises in the relationship between online and DP learning
	is the {\it threshold dimension}:
	a sequence $x_1,\ldots, x_k\in X$ is {\it threshold-shattered} by $\cH$
	if there are $h_1,\ldots,h_k\in \cH$ such that $h_i(x_j) = 1$ if and only if $i\leq j$ for all~$i,j\leq k$.
	The {\it threshold dimension}, $T(\cH)$ is the maximum size of a sequence that is threshold-shattered by $\cH$.
	The threshold dimension plays a key role in showing that DP learnable classes
	have a finite Littlestone dimension~\citep{AlonLMM19}. A classical theorem by~\cite{Shelah78classification} 
	in model theory shows that the Littlestone and the threshold dimensions are exponentially related.\footnote{
	The threshold dimension may be interpreted as a combinatorial abstraction of the geometric notion of {\it margin}.
	Under this interpretation, Shelah's result may be seen as an extension of the classical Perceptron's mistake-bound analysis by~\cite{perceptron}.} In particular $\ldim{\cH}<\infty$ if and only if $T(\cH)<\infty$.
	(See \Cref{thm:shelah} in the preliminaries section.)
	We prove the following closure theorem in terms of the threshold dimension.

\begin{theorem}[A Closure Theorem for the Threshold Dimension]\label{thm:thresholdscomp}
Let $G:\{0,1\}^k\to\{0,1\}$ be a boolean function, let~$\cH_1,\ldots,\cH_k\subseteq\{0,1\}^X$ be classes, and let $t\in\mathbb{N}$
such that~$T(\cH_i) < t$ for every $i\leq k$. Then,
\[T\bigl(G(\cH_1,\ldots,\cH_k)\bigr) <  2^{4 k4^{k}\cdot t}.\]
Moreover, an exponential dependence in $t$  is necessary:
	for every~$t \geq 6$ there exists a class $\cH$ such that $T(\cH)\leq t $ 
	and 
	\[T\Bigl(\{h_1\lor h_2 : h_1,h_2\in\cH\}\Bigr) \geq  2^{\lfloor t/5 \rfloor}.\] 
\end{theorem}

Note that the bounds in \Cref{thm:littlestone} and \Cref{thm:thresholdscomp} escalate rapidly with $k$ (the arity of $G$) and with~$t$. 
	It will be interesting to determine tight bounds.

\vspace{2mm}
	
By \cite{AlonLMM19,Bun20Privateonline}, \Cref{thm:littlestone} also implies closure properties for DP-learnable classes.
	However,  the quantitative bounds are even worse:
	not only do the bounds on the Littlestone dimension of $G(\cH_1,\ldots, \cH_k)$ escalate rapidly with $d$ and $k$,
	also the quantitative relationship between the Littlestone dimension and DP-learning sample complexity
	is very loose, and the best bounds exhibit a tower-like gap between the upper and lower bounds.
	For example, if the class of functions $\cH$ is finite and its Littlestone dimension is $\omega(\log\log |\cH|)$, then
	the bound of \Cref{thm:littlestone} is most likely to be much worse than the generic application of the exponential 
	mechanism, whose sample complexity is the logarithm of the size of the class. 
	We therefore explore the closure properties  of differentially-private learning algorithms directly
	and  derive the following bound. 

\begin{theorem}[A Closure Theorem for Private Learning (informal)]\label{thm:private}
Let $G:\{0,1\}^k\to\{0,1\}$ be a boolean function.
Let~$\cH_1,\ldots,\cH_k\subseteq\{0,1\}^X$ be classes 
that are $(\epsilon,\delta)$-differentially private and $(\alpha,\beta)$-accurate learnable   with sample complexity $m_i$ respectively. 
Then, $G(\cH_1,\ldots,\cH_k)$ is $(\eps,\delta)$-private and $(\alpha,\beta)$-accurate learnable with sample complexity
$$ \tilde{O}\left( \sum_{i=1}^k m_i\right)\cdot \poly(k,1/\epsilon,1/\alpha,\log(1/\beta)).$$
\end{theorem}	

The exact quantitative satement of the results appears in \Cref{thm:privacy}.
We remark that closure properties for {\em pure}
differentially-private learning algorithms (i.e., when $\delta=0$) are implied by the characterization of~\citep{Beimel19Pure}. Similarly, closure properties for {\em non-private} PAC learning are implied by
the characterization of their sample complexity in terms of the VC dimension and by the Sauer-Shelah-Perles Lemma~\citep{Sauer72lemma}. However, since there is no tight characterization of the sample complexity of approximate differentially-private learning algorithms (i.e., when $\delta>0$), we prove \cref{thm:private} algorithmically by constructing a (non-efficient) learning algorithm for $G(\cH_1,\ldots,\cH_k)$
from private learning algorithms for $\cH_1,\ldots,\cH_k$.

\medskip

\cite{BeimelNS15} proved that any {\em proper} private learning 
algorithm in the realizable case\footnote{That is,
when the examples are labeled by some $h \in \cH$.} can be transformed into an agnostic\footnote{That is, when the examples are labeled arbitrarily and the goal is to find a hypothesis whose error is close to the smallest error of a hypothesis in $\cH$.} private learning algorithm, with only a mild increase in the sample complexity. 
We show that the same result holds even for {\em improper} private learning (i.e., when the private learning algorithm can return an arbitrary hypothesis). 
\begin{theorem}[Private Learning Implies Agnostic Private Learning]
\label{thm:agnostic}
For every $0<\alpha,\beta,\delta<1$, every $m\in\N$, and every concept class $ \cH$, if there exists a 
$(1,\delta)$-differentially private $(\alpha,\beta)$-accurate PAC learner for the hypothesis class $ \cH$
with sample complexity $m$, then there exists an $(O(1),O(\delta))$-differentially private $(O(\alpha),O(\beta+\delta n))$-accurate {\em agnostic} learner for  $ \cH$ with sample complexity
$$n=O\Bigg(m + \frac{1}{\alpha^2}\left(\VC( \cH)+\log\frac{1}{\beta}\right)\Bigg).$$
Furthermore, if the original learner is proper, then the agnostic learner is proper.
\end{theorem}
We obtain this result by showing that   
a variant of the transformation of \citep{BeimelNS15} also works for the improper case; we do not know if the original 
 transformation of \citep{BeimelNS15} also works for the improper case. Our analysis of the transformation for the improper case is more involved than  the analysis for the proper case.

\subsection{Technical Overview}\label{sec:technical}

\subsubsection{Closure for Littlestone Dimension} 

Our proof of \Cref{thm:littlestone} exploits tools from online learning. 
It may be instructive to compare \Cref{thm:littlestone} with an analogous result for VC classes:
    a classical result by \cite{dudley1978} upper bounds the VC dimension of $G(\cH_1,\ldots,\cH_k)$ 
	by $\tilde O(d_1+\cdots+ d_k)$, where $d_i$ is the VC dimension of $\cH_i$.
	The argument uses the Sauer-Shelah-Perles Lemma~\citep{Sauer72lemma} to bound the growth-rate 
	(a.k.a.\ shatter function) of $G(\cH_1,\ldots, \cH_k)$ by some $n^{d_1+\cdots + d_k}$:
	indeed, if we let 
	$n=\VC(G(\cH_1,\ldots, \cH_k)),$
	then by the definition of the shatter function,
	$2^n \leq n^{d_1+\cdots + d_k}$,
	which implies that $n=\tilde O(d_1+\cdots +d_k)$ as stated.
    	It is worth noting that a notion of growth-rate as well as 
	a corresponding variant of the Sauer-Shelah-Perles Lemma also exist for Littlestone classes~\citep{bhaskar2017thicket,chase2018banned}. 
	However we are not aware of a way of using it to prove \Cref{thm:littlestone}.

We take a different approach. 
	We first focus on the case where $G$ is a majority-vote.
	That is, the class $\cH=G(\cH_1,\ldots,\cH_k)$ consists of all $k$-wise majority-votes of experts $h_i\in\cH_i$.
	We bound the Littlestone dimension of $\cH$ by exhibiting an online learning algorithm $A$ 
	that learns $\cH$ in the mistake-bound model with at most $\tilde O(k^2\cdot d)$ mistakes.
	The derivation of $A$ exploits fundamental tools from online learning such as
	the {\it Weighted Majority Algorithm} by \cite{LittlestoneWarmuth94} and {\it Online Boosting}~\citep{Chen2012online,beygelzimmer15optimal,Brukhim20online}.

Then, the bound for a general $G:\{0,1\}^k\to\{0,1\}$ is obtained
	by expressing $G$ as a formula which only uses majority-votes and negations gates.
	The exponential dependence in $k$ in the final bound is a consequence 
	of the formula-size which can be exponential in $k$.
	We do not know whether this exponential dependence is necessary.

\subsubsection{Closure for Threshold Dimension}

Our proof of \Cref{thm:thresholdscomp} is combinatorial. 
	First, note that an inferior bound follows from \Cref{thm:littlestone},
	using the fact that the Littlestone and threshold dimensions are exponentially related (see \Cref{thm:shelah}).
	However this approach yields a super-exponential bound on $T(G(\cH_1,\ldots, \cH_k))$.

The bound in \Cref{thm:thresholdscomp} follows by arguing contra-positively that if $T(G(\cH_1,\ldots, \cH_k))$ is large then
	$T(\cH_i)$ is also ``largish'' for some $i\leq k$.
	Specifically, if $T(G(\cH_1,\ldots, \cH_k))\geq \exp(t\exp(k))$ 
	then $T(\cH_i)\geq t$ for some~$i\leq k$.
	This is shown using a Ramsey argument that asserts that any large enough sequence $x_1,\ldots, x_n$
	that is threshold-shattered by $G(\cH_1\ldots \cH_k)$ must contain a relatively large subsequence
	that is threshold-shattered by one of the $\cH_i$'s. Quantitatively, if $n\geq \exp(t\exp(k))$
	then there must be a subsequence $x_{j_1},\ldots, x_{j_t}$ that is threshold-shattered by one of the $\cH_i$'s.

This upper bounds $T(G(\cH_1,\ldots, \cH_k))$ by some $\exp(t\exp(k))$, where $t=\max_i T(\cH_i)$.
	It is worth noting that, in contrast with \Cref{thm:littlestone}, an exponential dependence here is inevitable:
	we prove in \Cref{thm:thresholdscomp} that for any $t$ there exists a class $\cH$ with $T(\cH)\leq t$
	such that $T(\{h_1\lor h_2 : h_1,h_2\in \cH\})\geq \exp(t)$.
	This lower bound is achieved by a randomized construction.

\subsubsection{Private learning Implies Agnostic Private Learning}


We start by describing the transformation of \citep{BeimelNS15} from a proper private learning 
algorithm of a class $\cH$ to  an agnostic proper private learning 
algorithm for $\cH$. 
Assume that there is a private learning algorithm $\AAA$ for $\cH$ with sample complexity $m$.
The transformation takes a  sample $S$ of size $O(m)$ and constructs all possible behaviors $H$ of functions 
in $\cH$ on the points of the sample (ignoring the labels).
By the Sauer-Shelah-Perles Lemma, the number of such behaviors is at most $\left(\frac{e|S|}{\VC(\cH)}\right)^{\VC(\cH)}$. 
Then, it finds using the exponential mechanism a behavior
$h' \in H$ that minimizes the empirical error on the sample. (The exponential mechanism is guaranteed to identify a behavior with small empirical error because the number of possible behaviors is relatively small.)
Finally, the transformation relabeles the sample $S$ using $h'$  and applies $\AAA$ on the relabeled sample. If $\AAA$ is a proper learning algorithm then,
by standard VC arguments, the resulting algorithm is an agnostic algorithm for $\cH$. The privacy guarantees of the resulting algorithm are more delicate, and it is only $O(1)$-differentially private, even if $\AAA$ is $\epsilon$-differentially private for a small $\epsilon$. (The difficulty in the privacy analysis is the set of behaviors
$H$ is {\em data-dependent}. Therefore, 
the privacy guarantees of the resulting algorithms {\em are not} directly implied by those of the exponential mechanism, which assume that the set of possible outcomes is fixed and data-independent.)

When $\AAA$ is improper, we cannot use VC arguments to argue that the resulting algorithm is an agnostic learner.
We rather use the generalization properties of differential privacy (proved in~\citep{DworkFHPRR15,BassilyNSSSU16,RogersRST16,FeldmanS17,NSunpublished,JungLN0SS20}):
if a differentially private algorithm has a small empirical error on a sample chosen i.i.d.\ from some distribution, then it also has a small generalization error on the underlying distribution (even if the labeling hypothesis is chosen after seeing the sample). There are technical issues in applying  these results in our case that require some modifications in the transformation.   

\subsubsection{Closure for Differentially Private Learning}

We prove \Cref{thm:private} by constructing a private algorithm $\AlgClosureLearn$ 
for the class $G(\cH_1,\ldots,\cH_k)$ using private learning algorithms for the classes $\cH_1,\ldots,\cH_k$. Algorithm $\AlgClosureLearn$ uses the relabeling procedure (the one that we use to transform a private PAC learner into a private agnostic learner) in a new setting.

The input to  $\AlgClosureLearn$ is a sample labeled  by some function in $G(\cH_1,\ldots,\cH_k)$. 
The algorithm  finds hypotheses 
$h_1,\ldots,h_k$ in steps, where in  the $i$'th step, the algorithm finds a hypothesis $h_i$ such that 
$h_1,\dots,h_i$ have a  completion $c_{i+1},\dots,c_k$ to a hypothesis
$G(h_1,\dots,h_i,c_{i+1},\dots,c_k)$ with small error (assuming that $h_1,\dots,h_{i-1}$ have a good completion).

Each step of $\AlgClosureLearn$  is similar to the algorithm for agnostic learning described above. That is, in the $i$'th step, $\AlgClosureLearn$ first relabels the input sample $S$ using some $h \in \cH_i$ in a way that guarantees completion to a hypothesis with small empirical error. 
The relabeling $h$ is chosen using the exponential mechanism with an appropriate score function.
The relabeled sample is then fed  to the private algorithm for the class $\cH_i$ to produce a hypothesis $h_i$
and then the algorithm proceeds to the next step $i+1$.
As in the algorithm for agnostic learning, the proof that the hypothesis $G(h_1,\ldots,h_k)$ returned by the algorithm is easier when the private algorithms for $\cH_1,\ldots,\cH_k$ are proper and it is more involved if they are improper.

\section{Preliminaries}
\label{sec:Preliminaries}

This section is organized as follows:
\Cref{sec:ldprel} contains basic definitions and tools related to  the Littlestone dimension 
and \Cref{sec:PrelimPrivacy} contains basic definitions and tools  related to private learning.

\subsection{Preliminaries on the Littlestone Dimension}
\label{sec:ldprel}
The Littlestone dimension is a combinatorial 
parameter that characterizes regret bounds in online learning \citep{Littlestone87online,Bendavid09agnostic}. 
The definition of this parameter uses the notion of {\it mistake-trees}:
these are binary decision trees whose internal nodes are labeled by elements of $X$.
Any root-to-leaf path in a mistake tree can be described as a sequence of examples 
$(x_1,y_1),\ldots,(x_d,y_d)$, where $x_i$ is the label of the $i$'th 
internal node in the path, and $y_i=$ if the $(i+1)$'th node  
in the path is the right child of the $i$'th node, and otherwise $y_i = 0$.
We say that a tree $T$ is {\it shattered }by $\cH$ if for any root-to-leaf path
$(x_1,y_1),\ldots,(x_d,y_d)$ in $T$ there is $h\in \cH$ such that $h(x_i)=y_i$, for all $i\leq d$.
The Littlestone dimension of $\cH$, denoted by $\ldim{\cH}$, is the depth of the largest
complete tree that is shattered by~$\cH$.

\begin{definition}[Subtree]\label{def:subtree}
Let $T$ be labeled binary tree.
We will use the following notion of a {\it subtree} $T'$ of depth $h$ of $T$ by induction on $h$:
\begin{enumerate}
\item	 Any leaf of $T$ is a subtree of height $0$. 
\item	 For $h \geq 1$ a subtree of height $h$ is obtained from an internal vertex of $T$ 
	 together with a subtree of height $h-1$ of the tree rooted at its left child 
	 and a subtree of height $h-1$ of the tree rooted at its right child. 
\end{enumerate}
\end{definition}
	 Note that if $T$ is a labeled tree and  it is shattered by the class $\cH$, 
	 then any subtree $T'$ of it  with the same labeling of its internal vertices is shattered by the class $\cH$.

\paragraph{Threshold Dimension.}
A classical theorem of Shelah in model-theory connects bounds on 2-rank (Littlestone dimension) to the concept of {\it thresholds}:
	let $\cH\subseteq \{0,1\}^X$ be a hypothesis class. 
	We say that a sequence $x_1,\ldots, x_k\in X$ is {\it threshold-shattered} by $\cH$
	if there are $h_1,\ldots,h_k\in \cH$ such that $h_i(x_j) = 1$ if and only if $i\leq j$ for all~$i,j\leq k$.
	Define the {\it threshold dimension}, $T(\cH)$, as the maximum size of a sequence that is threshold-shattered by $\cH$.
   
\begin{theorem}[Littlestone Dimension versus Threshold Dimension \citep{Shelah78classification, Hodges97book}]\label{thm:shelah}
Let $\cH$ be a hypothesis class, then:
\[T(\cH) \geq \lfloor \log \ldim{\cH}\rfloor ~~\text{  and  } ~~\ldim{\cH} \geq \lfloor \log T(\cH)\rfloor.\]
\end{theorem}

\subsection{Preliminaries on Private Learning}
\label{sec:PrelimPrivacy}

\paragraph{Differential Privacy.} 
Consider a database where each record contains information of an individual. An algorithm is said to preserve differential privacy if a change of a single record of the database (i.e., information of an individual) does not significantly change the output distribution of the algorithm. Intuitively, this means that the information inferred about an individual from the output of a differentially-private algorithm is similar to the information that would be inferred had the individual's record been arbitrarily modified or removed. Formally:

\begin{definition}[Differential privacy~\citep{DworkMNS06,DworkKMMN06}] \label{def:dp} 
A randomized algorithm $\Alg$ is $(\epsilon,\delta)$-differentially private if for all neighboring databases $\db_1,\db_2\in X^m$ (i.e., databases differing in one entry), and for all sets $\mathcal{F}$ of outputs,
\begin{equation}
\label{eqn:diffPrivDef}
   \Pr[\Alg(\db_1) \in \mathcal{F}] \leq \exp(\epsilon) \cdot \Pr[\Alg(\db_2) \in \mathcal{F}] + \delta,  
\end{equation}
where the probability is taken over the random coins of $\Alg$. 
When $\delta=0$ we omit it and say that $\Alg$ preserves {\em pure} $\epsilon$-differential privacy.
When $\delta>0$, we use the term  {\em approximate} differential privacy , in which case $\delta$ is typically a negligible function of the database size $m$.
\end{definition}

\paragraph{PAC Learning.}
We next define the probably approximately correct (PAC) model of~\cite{Valiant84}.
A  hypothesis $c:X\rightarrow \{0,1\}$ is a predicate that labels {\em examples} taken from the domain $X$ by either 0 or 1. 
We sometime refer to a hypothesis as a concept.
A \emph{hypothesis class} $\cH$ over $X$ is a set of hypotheses (predicates) mapping $X$ to $\{0,1\}$. 
A learning algorithm is given examples sampled according to an unknown probability distribution $\DDD$ over $X$, 
and labeled according to an unknown {\em target} concept $c\in \cH$. 
The learning algorithm is successful when it outputs a hypothesis $h$ that approximates the target concept over samples from $\DDD$. 
More formally:

\begin{definition}
The {\em generalization error} of a hypothesis $h:X\rightarrow\{0,1\}$ with respect to a concept $c$ and a distribution $\DDD$ over $X$ is defined as 
$\error_{\DDD}(c,h)=\Pr_{x \sim \DDD}[h(x)\neq c(x)].$
If $\error_{\DDD}(c,h)\leq\alpha$ we say that $h$ is {\em $\alpha$-good} for $c$ and $\DDD$.
\end{definition}

\begin{definition}[PAC Learning~\citep{Valiant84}]\label{def:PAC}
An algorithm $\Alg$ is an {\em $(\alpha,\beta)$-accurate PAC learner} for a hypothesis
class $\cH$ over $X$ if for all concepts $c \in \cH$, all distributions $\DDD$ on $X$,
given an input of $m$ samples $\db =(z_1,\ldots,z_m)$, where $z_i=(x_i,c(x_i))$ and each $x_i$
is drawn i.i.d.\ from $\DDD$, algorithm $\Alg$ outputs a
hypothesis $h$ satisfying
$$\Pr[\error_{\DDD}(c,h)  \leq \alpha] \geq 1-\beta,$$
where the probability is taken over the random choice of
the examples in $\db$ according to $\DDD$ and the random coins  of the learner $\Alg$.
If the output hypothesis $h$ always satisfies $h\in \cH$ then $\Alg$ is called a {\em proper} PAC learner; otherwise, it is called an {\em improper} PAC learner.
\end{definition}

\begin{definition}
For an unlabeled sample $\db=(x_i)_{i=1}^m$, the {\em empirical error} of two concepts $c,h$ is
$\error_S(c,h) = \frac{1}{m} |\{i : c(x_i) \neq h(x_i)\}|.$
For a labeled sample $\db=(x_i,y_i)_{i=1}^m$, the {\em empirical error} of $h$ is
$\error_S(h) = \frac{1}{m} |\{i : h(x_i) \neq y_i\}|.$
\end{definition}

The previous definition of PAC learning captures the realizable case, that is,  the examples are drawn from some distribution and labeled according to some concept $c \in \cH$. We next define agnostic learning, i.e., where there is a distribution over labeled examples and the goal is to find a hypothesis whose error is close to the error of the best hypothesis in $\cH$ with respect to the distribution. Formally, 
for a distribution $\mu$ on $X\times \ourset{0,1}$ and a function $f:X\rightarrow\ourset{0,1}$ we define $\error_{\mu}(f)=\Pr_{(x,a) \sim \mu}[f(x)\neq a]$.
\begin{definition}[Agnostic PAC Learning]\label{def:agnosticPAC}
Algorithm $\Alg$ is an $(\alpha,\beta)$-accurate {\em agnostic PAC learner} for a hypothesis
class $\cH$ with sample complexity $m$ if for all distributions $\mu$ on $X\times\ourset{0,1}$,
given an input of $m$ labeled samples $\db =(z_1,\ldots,z_m)$, where each labeled example $z_i=(x_i,a_i)$ 
is drawn i.i.d.\ from $\mu$, algorithm $\Alg$ outputs a
hypothesis $h\in \cH$ satisfying
$$\Pr\left[\left\lvert \error_{\mu}(h)-\min_{c \in \cH} \ourset{\error_{\mu}(c)}\right\rvert \leq \alpha\right] \geq 1-\beta,$$
where the probability is taken over the random choice of
the examples in $\db$ according to $\mu$ and the random coins  of the learner $\Alg$.
If the output hypothesis $h$ always satisfies $h\in \cH$ then $\Alg$ is called a {\em proper} agnostic PAC learner; otherwise, it is called an {\em improper} agnostic PAC learner.
\end{definition}

The following bound is due to~\citep{Vapnik71uniform,BlumerEHW89}.
\begin{theorem}[VC-Dimension Generalization Bound]\label{thm:VCconsistant}
Let $\cH$ and $\DDD$ be a concept class and a distribution over a domain $X$.
Let $\alpha,\beta>0$, and $$m\geq\frac{80}{\alpha}\left(\VC(\cH)\ln\left(\frac{16}{\alpha}\right)+\ln\left(\frac{2}{\beta}\right)\right).$$
Suppose that we draw an unlabeled  sample $S=(x_i)_{i=1}^m$, where $x_i$ are drawn i.i.d.\ from $\DDD$. Then,
$$
\Pr[\exists c,h\in \cH \text{ s.t. } \error_{\DDD}(h,c)>\alpha \; \wedge \; \error_S(h) < \alpha/2]\leq\beta.
$$
\end{theorem}


The next theorem, due to
\citep{Vapnik71uniform,Anthony2009,Anthony93}, handles (in particular) the agnostic case. 
\begin{theorem}[VC-Dimension Agnostic Generalization Bound]\label{thm:VCagnostic}
There exists a constant $\gamma$ such that for every  domain $X$, every concept class $\cH$ over the domain $X$, and every distribution  $\mu$ over the domain $X \times\ourset{0,1}$: For a sample $S=(x_i,y_i)_{i=1}^m$ where $$m\geq\gamma \frac{ \VC(\cH)+\ln(\frac{1}{\beta})}{\alpha^2}$$
and $\{(x_i,y_i)\}$ are drawn i.i.d.\  from $\mu$, it holds that
$$\Pr\Big[\exists \; h\in \cH  \text{ s.t. }  \big|\error_\mu(h)-\error_S(h)\big|\geq\alpha\Big]\leq\beta.$$  
\end{theorem}

Notice that in \Cref{thm:VCagnostic} the sample complexity is proportional to $\frac{1}{\alpha^2}$, as opposed to $\frac{1}{\alpha}$ in
\Cref{thm:VCconsistant}.

\medskip



\paragraph{Private Learning.}
Consider an algorithm $\AAA$ in the probably approximately correct (PAC) model of~\cite{Valiant84}. We say that $\AAA$ is a {\em private} learner if it also satisfies differential privacy w.r.t.\ its training data. 
\begin{definition}[Private PAC Learning~\citep{KasiviswanathanLNRS11}]
Let $\Alg$ be an algorithm that gets an input $\db =(z_1,\ldots,z_m)${, where each $z_i$ is a labeled example}. Algorithm $\Alg$ is an {\em $(\epsilon,\delta)$-differentially private $(\alpha,\beta)$-accurate PAC learner 
with sample complexity $m$} for a 
class $\cH$ over $X$ 
if
\begin{description}
\item{\sc Privacy.} 
Algorithm $\Alg$ is $(\epsilon,\delta)$-differentially private (as in  Definition~\ref{def:dp});
\item{\sc Utility.}
and Algorithm $\Alg$ is an {\em $(\alpha,\beta)$-accurate PAC learner} 
for $\cH$ with sample complexity $m$ {(as in  Definition~\ref{def:PAC})}.
\end{description}
When $\delta=0$ (pure privacy) we omit it from the list of parameters.
\end{definition}

Note that the utility requirement in the above definition is an average-case requirement, as the learner is only required to do well on typical samples. In contrast, the privacy requirement is a worst-case requirement that must hold for every pair of neighboring databases (no matter how they were generated).

The following definition and lemma are taken from \cite{BunNSV15}.
\begin{definition}[Empirical Learner] Algorithm $\Alg$ is an 
$(\alpha, \beta)$-accurate empirical learner for a
class $\cH$ over $X$ with sample complexity $m$ if for every $c \in \cH$
and for every sample 
$S$ of size $m$ that is labeled by $c$, the algorithm $\Alg$ outputs a
hypothesis $h \in H$ satisfying 
\[\Pr[\error_S(c, h) \leq \alpha] \geq 1 - \beta.\]
\end{definition}

\begin{lemma}[\cite{BunNSV15}]
\label{lem:empirical}
 Suppose $\Alg$ is an $(\epsilon, \delta)$-differentially private $(\alpha, \beta)$-accurate PAC learner for a concept
class $\cH$ with sample complexity $m$. Let $\Alg'$ be an algorithm, whose input sample $S$ contains $9m$ randomly labeled examples.
Further assume that $\Alg'$ samples with repetitions $m$ labeled examples from $S$ and returns the output of $\Alg$ on these examples.
Then, $\Alg'$ is an $(\epsilon, \delta)$-differentially private $(\alpha, \beta)$-accurate
empirical learner for $\cH$ with sample complexity $9m$. Clearly, if $\Alg$ is proper, then so is $\Alg'$.
\end{lemma}

\paragraph{The Exponential Mechanism.}
We next describe the exponential mechanism of
\cite{McSherryT07}.
Let $X$ be a domain and $H$ a set of solutions.
Given a score function $q:X^*\times H \rightarrow\N$, and a database $S\in X^*$, the goal is to chooses a solution $h\in H$ approximately minimizing $q(S,h)$. The mechanism chooses a solution probabilistically, where the probability mass that is assigned to each solution $h$ decreases exponentially with its score $q(S,h)$:


\begin{algorithm}
\caption{$\Alg_{\rm ExponentialMechanism}$}
\label{alg:ExponentialMechanism}
{\bf Input:} parameter $\epsilon$, finite solution set $H$, database $S\in X^m$, and a sensitivity 1 score function~$q$ (i.e., $|q(D)-q(D')| \leq 1$ for
every neighboring $D,D' \in X^m$).
\begin{enumerate}
	\item Randomly choose $h \in H$ with probability
	$\frac{\exp\left(-\epsilon \cdot q(S,h) /2 \right)}{\sum_{f\in H}\exp\left(-\epsilon \cdot q(S,f) /2 \right)}.$
	\item Output $h$.
\end{enumerate}
\end{algorithm}

\begin{prop}[Properties of the Exponential Mechanism]\label{prop:ExpMech}
(i) The exponential mechanism is $\epsilon$-differentially private. (ii)
Let $\hat{e}\triangleq\min_{f\in H}\{q(S,f)\}$ and $\Delta>0$. The exponential mechanism outputs a solution $h$ such that $q(S,h)\geq(\hat{e} +\Delta m)$ with probability at most $|H| \cdot \exp(-\epsilon \Delta m /2)$.
\end{prop}

\cite{KasiviswanathanLNRS11} showed  that the exponential mechanism can be used as a generic private learner -- when used with the score function $q(S,h)=|\{i:h(x_i)\neq y_i\}|=m \cdot \error_S(h)$, the probability that the exponential mechanism outputs a hypothesis $h$ such that $\error_S(h)>\min_{f\in H}\{\error_S(f)\} + \Delta$ is at most $|H| \cdot \exp(-\epsilon \Delta m /2)$. This results in a generic private proper-learner for every finite concept class $\cH$, with sample complexity $O_{\alpha,\beta,\epsilon}(\log|\cH|)$.

\paragraph{Generalization Properties of Differentially Private Algorithms.}

In this paper we use the fact that differential privacy implies generalization \citep{DworkFHPRR15,BassilyNSSSU16,RogersRST16,FeldmanS17,NSunpublished,JungLN0SS20}: 
differentially private learning algorithms satisfy that their empirical loss is typically close to their population loss.
We use the following variant of this result, which is a multiplicative version that applies also to the case that $\epsilon >1$ (as needed in this paper).

\begin{theorem}[DP Generalization -- Multiplicative version  \citep{DworkFHPRR15,BassilyNSSSU16,FeldmanS17,NSunpublished}]  
\label{thm:GeneralizationMultiplicative}
Let $\Alg$ be an $(\epsilon,\delta)$-differentially private algorithm
that operates on a database of $S \in X^n$  and outputs a predicate $\testExample : X \rightarrow \ourset{0,1}$. Let $\DDD$ be a distribution over
$X$ and  $S$ be a database containing $n$ i.i.d.\  elements from $\DDD$. Then, 
\begin{align*}
\Pr_{\substack{S\inr X^n,\\\testExample\gets_R \Alg(S)}} &
\left[
 \E_{x \in_\DDD X} [\testExample(x)] > 
 e^{2\epsilon}\left( \frac{\sum_{x \in S} \testExample(x)}{n} +
\frac{10}
{\epsilon n}
\log\left(
\frac{1}
{\epsilon \delta n} \right)\right)
\right] 
< O\left(
\frac{\epsilon \delta n}
{\log(
\frac{1}{\epsilon \delta n})}
\right).
\end{align*}
\end{theorem}

\section{Closure of Littlestone Classes}\label{sec:littlestoneproof}

In this section we study closure properties for Littlestone classes.
    We begin in \Cref{sec:lsunion} with a rather simple (and tight) analysis of the behavior of the Littlestone 
    and Threshold dimension under unions. 
    Then, in \Cref{sec:lscomp} we prove our main results in this part (\Cref{thm:littlestone,thm:thresholdscomp}) 
	which bound the variability of the Littlestone and Thresholds dimension under arbitrary compositions.

\subsection{Closure Under Unions}\label{sec:lsunion}
We begin with two basic bounds on the variability of the Littlestone/Threshold dimension under union.
Note that here $\cH_1\cup \cH_2$ denotes the usual union: $\cH_1\cup \cH_2 = \{h : h\in \cH_1 \text{ or } h\in \cH_2\}$.
These bounds are useful as they allows us to reduce a bound on the dimension of $G(\cH_1,\cH_2)$ for arbitrary $\cH_1,\cH_2$
to the case where $\cH_1=\cH_2$ (because $G(\cH_1,\cH_2)\subseteq G(\cH,\cH)$ for $\cH=\cH_1\cup \cH_2$).
\begin{obs}\label{obs:uniont}[Threshold Dimension Under Union]
Let $\cH_1,\cH_2\subseteq\{0,1\}^X$ be hypothesis classes with $T(\cH_i)=t_i$.
Then, 
\[T(\cH_1 \cup \cH_2) \leq t_1 + t_2.\]
Moreover, this bound is tight: for every $t_1,t_2$, there are classes $\cH_1,\cH_2$
with Threshold dimension $t_1,t_2$ respectively such that $T(\cH_1 \cup \cH_2) = t_1 + t_2$.
\end{obs}
\begin{proof}
For the upper bound, observe that if $h_1\ldots h_m\in \cH_1\cup\cH_2$  threshold-shatters the sequence $x_1\ldots x_m$
then $\{h_i : h_i\in \cH_j\}$ threshold-shatters $\{x_i : h_i\in \cH_j\}$ for $j\in\{1,2\}$.
For the lower bound, set $X=[t_1+ t_2]$, $\cH_1 = \{ h_i : i\leq t_1\}$, and $\cH_2=\{h_i : t_1< i \leq t_1+ t_2\}$,
where $h_i(j)=1$ if and only if $i\leq j$.
\end{proof}

\begin{prop}[Littlestone Dimension Under Union]\label{prop:union}
Let $\cH_1,\cH_2\subseteq\{0,1\}^X$ be hypothesis classes with $\ldim{\cH_i}=d_i$.
Then, 
\[\ldim{\cH_1 \cup \cH_2} \leq d_1 + d_2 + 1.\]
Moreover, this bound is tight: for every $d_1,d_2$, there are classes $\cH_1,\cH_2$
with Littlestone dimension $d_1,d_2$ respectively such that $\ldim{\cH_1 \cup \cH_2} = d_1 + d_2 + 1$.
\end{prop}

\begin{proof}[Proof of \Cref{prop:union}]
There are several ways to prove this statement. One possibility is to use the realizable online mistake-bound setting~\citep{Littlestone87online}
	and argue that $\cH_1\cup\cH_2$ can be learned with at most~$d_1 + d_2 + 1$ mistakes in this setting.
	We present here an alternative inductive argument, which may be of independent interest.
	Towards this end, it is convenient to define the depth of the empty tree as $-1$, 
	and that of a tree consisting of one vertex (leaf) as $0$.

Consider a shattered tree $T$ of depth $d=\ldim{\cH_1 \cup \cH_2}$ with leaves labelled $\cH_1$ and~$\cH_2$ in the obvious way. 
	Recall the notion of a subtree in \Cref{def:subtree}, 
	and let $x\leq \ldim{\cH_1}$ be the maximum depth of a complete binary subtree all whose leaves are $\cH_1$ leaves, 
	and~$y\leq \ldim{\cH_2}$ the maximum depth of a subtree all whose leaves are $\cH_2$-leaves. 
	Similarly, let~$x_L,y_L$ denote the maximum depth of a $\cH_1$-subtree 
	and a $\cH_2$-subtree in the tree rooted at the left child of the root of $T$, 
	and let $x_R,y_R$ be the same for the tree rooted at the right child.

It suffices to show that $x+y \geq d-1$:
	clearly $x \geq \max (x_L,x_R)$ and also $x \geq \min(x_L,x_R)+1$ thus $x \geq (x_L+x_R)/2+1/2$.
	Similarly $y \geq (y_L+y_R)/2+1/2$, hence 
	\[x+y \geq \frac{x_L+y_L}{2}+\frac{x_R+y_R}{2}+1\] 
	and this gives by induction on $d$ (starting with $d=0$ or $1$) 	that $x+y \geq d-1$ as required.
	
To see that this bound is tight, pick $n\geq d_1+d_2+1$ and set 
\[\cH_1 = \Big\{ h:[n]\to\{\pm1\}\ : \sum_i h_i \leq d_1 \Bigr\}  ~~\text{ and }~~ \cH_1 = \Big\{ h:[n]\to\{\pm1\}\ : \sum_i h_i \geq n-d_2 \Bigr\}.\]
One can verify that $\ldim{\cH_i}=d_i$, for $i=1,2$ and that $\ldim{\cH_1 \cup \cH_2} = d_1 + d_2 + 1$, as required
(in fact, even the VC dimension of $\cH_1\cup\cH_2$ is $d_1+d_2 + 1$).
\end{proof}

\Cref{prop:union} implies that $\ldim{\cup_{i=1}^k \cH_i} =O (k\cdot d)$ provided that $\ldim{\cH_i}\leq d$ for al~$i$,
and that this inequality can be tight when $k=2$. 
The following proposition shows that for a larger $k$ this bound can be significantly improved:
\begin{prop}[Littlestone Dimension Under Multiple Unions]\label{prop:largeunions}
Let $\cH_1,\ldots,\cH_k\subseteq\{0,1\}^X$ be \new{hypothesis} classes with $\ldim{\cH_i}\leq d$.
Then, for every $0< \eps<1/2$,
\begin{align*}
\ldim{ \bigcup_{i=1}^k \cH_i} 
					    & \leq 3d + 3\log k.
\end{align*}
Moreover, this bound is tight up to a constant factor: 
for every $k$, there are classes $\cH_1,\ldots, \cH_k$
with $\ldim{\cH_i}\leq d$ such that $\ldim{\cup_i \cH_i} \geq  d+ \lfloor \log k \rfloor$.
\end{prop}
\Cref{prop:largeunions} demonstrates a difference with the threshold dimension. 
    Indeed, while the bound above scales logarithmically with $k$, 
    in the case of the threshold dimension a linear dependence in $k$ is necessary:
	indeed, set $X=[k\cdot t]$, $\cH_i = \{ h_j : (i-1)\cdot t < j i\leq i\cdot t\}$, where $h_i(j)=1$ if and only if $i\leq j$.
	Thus, $\ldim{\cH_i} = t$ for all $i$ and $\ldim{\cup_{i=1}^k\cH_i} = k\cdot t >> t + \log k$.
\begin{proof}[Proof of \Cref{prop:largeunions}]

We begin with the lower bound: pick any class $\cH\subseteq\{0,1\}^X$ with Littlestone dimension $d$,
	and let $T$ be a tree of depth $d$ which is shattered by $\cH$.
	Pick $\lfloor \log k\rfloor$ new points $z_1,\ldots, z_{\lfloor \log k\rfloor}\notin X$,
	and extend the domain $X$ to $X'=X\cup\{z_1\ldots, z_{\lfloor \log k\rfloor}\}$.
	Define $\cH'\subseteq \{0,1\}^{X'}$  by extending each $h\in \cH$ to the $z_i$'s in each of the $k'= 2^{\lfloor \log k\rfloor}$ possible ways.
	(So, each $h\in \cH$ has $k'$ copies in $\cH'$, one for each possible pattern on the $z_i$'s.)
	Thus, $\cH'$ is a union of $k'$ copies of $\cH$, one copy for each boolean pattern on the~$z_i$'s.
	In particular, $\cH'$ is the union of $k'$ classes with Littlestone dimension $d$.
	Also note that $\ldim{\cH'}\geq  \lfloor \log k \rfloor + d$, as witnessed by the tree which is illustrated in \Cref{fig:shatteredtree}.
	
\begin{figure}
\centering
\includegraphics[scale=0.3]{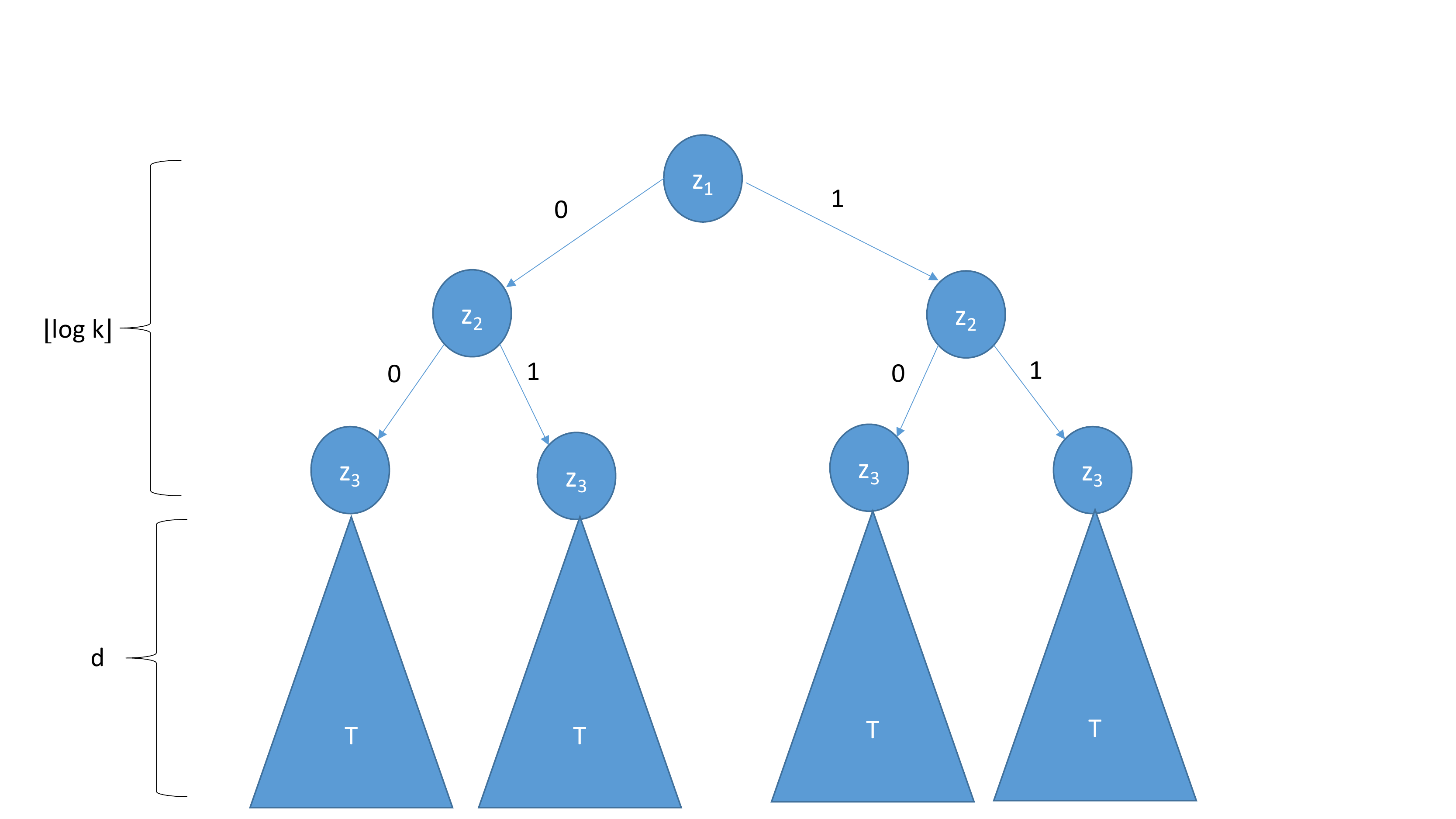}
\caption{\small{An illustration of the tree shattered by $\cH'$ in the construction in \Cref{prop:largeunions}. 
In this illustration~$\lfloor \log k\rfloor$ equals 3.}}\label{fig:shatteredtree}
\end{figure}

The upper bound is based on a multiplicative-weights argument.
	Recall that the Littlestone dimension equals the optimal number of mistakes performed
	by a deterministic online learner in the mistake-bound model 
	(i.e.\ online learning when the sequence of input examples is labelled by some $h\in\cH$). 
	Thus, it suffices to demonstrate an online learner for $\cup_{i=1}^k \cH_i$	which makes at most $3d+3\log k$ mistakes.
	Pick for every $\cH_i$ an online learner $A_i$ which makes at most $d$ mistakes on input sequences consistent with $\cH_i$.
	We set the online learning algorithm $A$ for $\cH=\cup_{i=1}^k \cH_i$ to be {\it The Weighted Majority Algorithm} by~\cite{LittlestoneWarmuth94}
	with the $k$ experts being the algorithms $A_1,\ldots, A_k$.
	Now, consider an input sequence $S=(x_1,y_1),\ldots (x_T,y_T)$ consistent with $\cH$.
	Thus, $S$ is consistent with $\cH_i$ for some $i\leq k$ and therefore $A_i$ makes at most $d$ mistakes on it.
	Thus, by the multiplicative weights analysis (see e.g.\ Corollary 2.1 in \cite{LittlestoneWarmuth94}),
	the number of mistakes $A$ makes on $S$ is at most 
	\[\frac{\log k  + d\log\frac{1}{\beta}}{\log\frac{2}{1+\beta}},\]
	where $0\leq \beta < 1$ is multiplicative factor which discounts the weights
	of wrong experts. The upper bound follows by setting $\beta=1/2$.

\end{proof}

\subsection{Closure Under Composition}\label{sec:lscomp}
%

\subsubsection{Threshold Dimension}\label{sec:thresholdscomp}
\paragraph{Proof of \Cref{thm:littlestone}.}
We begin with the upper bound.
Let $T(G(\cH_1,\ldots,\cH_k))=n$. It suffices to show that if $n\geq 2^{4 k4^{k}\cdot t}$
	then there is~$i\leq k$ such that $T(\cH_i)\geq t$.
	By assumption, there are $x_1,x_2\ldots x_n\in X$ and functions $h_{ij}\in\cH_j$,  for $1 \leq i \leq n, 1 \leq j \leq k$
	such that  
	\[(\forall i,j\leq n): G(h_{i1}, h_{i2}, \ldots ,h_{ik})(x_j) =1 \iff  i \leq j.\]

Construct a coloring of the edges of the complete graph on $[n]$ by $4^k$ colors as follows:
	for each $1 \leq p<q \leq n$, the color of the edge $\{p,q\}$ is
	given by the following ordered sequence of~$2k$ bits:
\begin{align*}
\bigl(&h_{p,1}(x_q), h_{p,2}(x_q), \ldots ,h_{p,k}(x_q),\\
 &h_{q,1}(x_p), h_{q,2}(x_p), \ldots ,h_{q,k}(x_p)\bigr).
\end{align*}
By Ramsey Theorem \citep{Ramsey30}, if $n \geq (4^k)^{2t\cdot 4^k} = 2^{4k4^{k}\cdot t}$ then there is a monochromatic set $A\subseteq [n]$ of size $\lvert A\rvert = 2t$.\footnote{We use here the following basic bound:
if $n\geq c^{r\cdot c}$, then for every coloring of the edges of the complete graph $K_n$ in $c$ colors there exists a monochromatic set of size $r$.
This follows, e.g.\ from Corollary 3 in~\cite{Greenwood55comb}.}
 	Denote the elements of $A$ by 
	\[A=\{i_1<j_1<i_2<j_2< \ldots <i_{t}<j_{t}\},\] 
	and let $u=(u_1\ldots u_k)$, $v=(v_1\ldots v_k)$ such that the color of every pair in $A$ is
\begin{align*}
(&v_1, v_2, \ldots ,v_k,\\
 &u_1, u_2, \ldots ,u_k).
\end{align*}
Thus, for every pair $p,q\leq d$ and every $r\leq k$:
	\[h_{i_p,r}(x_{j_q}) = 
	\begin{cases}
	v_r	& p\leq q\\
	u_r	& p > q.
	\end{cases}
	\]
We claim that $u\neq v$: 
	indeed,  $x_{j_1}, x_{j_2}, x_{j_3}, \ldots ,x_{j_t}$ is threshold-shattered
	by the functions
	\[G(h_{i_1,1},h_{i_1,2}, \ldots ,h_{i_1,k}),
	G(h_{i_2,1},h_{i_2,2}, \ldots ,h_{i_2,k}),
	  \cdots
	G(h_{i_{d},1},h_{i_{d},2}, \ldots ,h_{i_{t},k}).
	\]
	Thus,
	\begin{align*}
	p\leq q &\implies G(v) = G(h_{i_p,1}\ldots h_{i_p,k})(x_q) = 1,\\
	p > q	   &\implies G(u) = G(h_{i_p,1}\ldots h_{i_p,k})(x_q) = 0.
	\end{align*}
	Therefore, $v\in G^{-1}(0)$ and $u\in G^{-1}(1)$ and in particular $u\neq v$. 
	Pick an index $r$ so that~$u_r \neq v_r$.  
	Therefore, for every $p,q\leq t$:
	\[h_{i_p,r}(x_{j_q}) = 
	\begin{cases}
	v_r	& p\leq q\\
	u_r	& p > q,
	\end{cases}
	\text{and}~~~ v_r\neq u_r.
	\]
	This shows that either $x_1\ldots x_t$ is threshold shattered by $\cH_r$ (if $v_r=1, u_r=0$),
		or $x_t\ldots x_1$ is thresholds shattered by $\cH_r$  (if $v_r=0, u_r=1$);
		in either way, the threshold dimension of $\cH_r$ is at least $t$. This completes the proof of the upper bound.
		
\paragraph{Lower Bound.}
We next prove the lower bound.
Let $m= 2^{\lfloor t/5\rfloor}$,
and construct $\cH\subseteq \{0,1\}^m$ randomly as follows:
$\cH$ consists of $2m$ random functions 
\[\cH= \{f_1\ldots f_m,g_1\ldots g_m\},\]
where for each $i$ set $f_i(j)=g_j(j)=0$ for $j> i$,
and for $j\leq i$, pick uniformly at random one of $f_i,g_i$,
set it to be $1$ in position $j$ and set the other to be $0$ in position $j$.
All of the above~${m-1 \choose 2}$ random choices are done independently.
By construction, $\{h_1\lor h_2 : h_1,h_2\in \cH\}$ threshold-shatters the sequence $1,2\ldots,m$ with probability $1$ and hence has threshold dimension at least~$m$.
It suffices to show that with a positive probability it holds that 
\begin{equation}\label{eq:suffices}
T(\cH) \leq 2k,
\end{equation} 
where $k=(2+\frac{1}{\log m})\log m = 2\lfloor t/5\rfloor + 1$. Indeed, $2k =  4\lfloor t/5\rfloor + 2 \leq t$ whenever $t\geq 6$.

\vspace{2mm}

We set out to prove \Cref{eq:suffices}.
Consider the following event:
\begin{center}
{\it $\mathcal{E}:=$ There exist no $x_1,\ldots,x_k\in [m]$, $h_1,\ldots h_k \in\cH$ such that~$h_i(x_j)=1$ for all $i,j\leq k$}.
\end{center}
Note that $\mathcal{E}$ implies that~$T(\cH)\leq 2k$
and therefore it suffices to show that $\Pr[E] > 0$.
Towards this end we use a union bound:
we define a family of ``bad'' events whose total sum of probabilities is less than one
with the property that if none of the bad events occurs then~$\mathcal{E}$ occurs.
The bad events are defined as follows:
for any pair of subsets $A,B\subseteq [m]$ of size $\lvert A\rvert =\lvert B\rvert = k$,
let $\mathcal{B}_{A,B}$ denote the event
\begin{center}
{\it $\mathcal{B}_{A,B}:=$  ``For every $i\in A$ there exists $r_i\in\{f_i,g_i\}$ such that $r_i(j) = 1$ for all $j\in B$.''}
\end{center}
Note that indeed $\lnot\mathcal{E}$ implies $\mathcal{B}_{A,B}$ for some $A,B$ and thus it suffices to show that 
with a positive probability none of the $\mathcal{B}_{A.B}$ occurs.
We claim that
\[\Pr[\mathcal{B}_{A,B}] \leq 2^{-k(k-1)}.\]
Indeed, for a fixed $i\in A$, the probability that one of $f_i,g_i$ equals to $1$ on all $j\in B$
is at most~$2^{-(k-1)}$. By independence, the probability that the latter simultaneously holds
for every $i\in A$ is at most $2^{-k(k-1)}$.
Thus, the probability that $\mathcal{B}_{A,B}$ occurs for at least one pair $A,B$ is at most
\[{m \choose k}^2 2^{-k(k-1)} < 2^{2k\log m - k(k-1)} \leq 1,\]
where the last inequality holds because $k=(2+\frac{1}{\log m})\log m$.
%

\qed

\subsubsection{Littlestone Dimension}\label{sec:littlestonecomp}
\paragraph{Proof of \Cref{thm:littlestone}.}

We will first show that for an odd $k$, the majority-vote $G=\maj_k$ satisfies
\begin{equation}\label{eq:boostmaj}
\ldim{\maj_k(\cH_1\ldots \cH_k)} \leq \tilde O(k^2\cdot d).
\end{equation}
(Recall that $d=\max_i \ldim{\cH_i}$.)
Then, we use this to argue that for any $G$,
\begin{equation}\label{eq:formula}
\ldim{G(\cH_1\ldots \cH_k)} \leq \tilde O(2^{2k}k^2d).
\end{equation}

We start with proving \Cref{eq:boostmaj}. Let $\cH = \cup_{i=1}^k \cH_i$ and $\cH_k = \maj_k(\cH,\ldots,\cH)$.
Since $\maj_k(\cH_1,\ldots,\cH_k)\subseteq \cH_k$, it suffices to show that $\ldim{\cH_k} \leq \tilde O(k^2d)$. 
We use online boosting towards this end.

Online boosting (in the realizable setting) is an algorithmic framework which allows to transform a weak online learner for $\cH$
with a non-trivial mistake-bound of $({1}/{2} - \gamma)T + R(T)$, where $R(T)=o(T)$ is a sublinear regret function,
to a strong online learner with a vanishing mistake-bound of $O(R(T)/\gamma)$.
Online boosting has been studied by several works (e.g.~\cite{Chen2012online,beygelzimmer15optimal,Brukhim20online}). 
We use here the variant given by~\cite{Brukhim20online} (see Theorem 2 there)\footnote{The bound in Theorem 2 in \citep{Brukhim20online} contains an additional term which depends on $N$, the number of copies of the weak learner which are used by the boosting algorithm. Since here we are only concerned with the number of mistakes, we can eliminate this term by letting $N\to\infty$.}.

Which weak learner to use? Recall that by~\cite{Bendavid09agnostic} (see Equation \Cref{eq:ldimonline}) 
there exists an \underline{agnostic} online learning algorithm $W$ for $\cH$ whose (expected) regret bound is 
\[R(T)=O(\sqrt{\ldim{\cH}T\log T}).\]
We claim that $W$ is a weak learner for $\cH_k$ with mistake-bound 
\begin{equation}\label{eq:weaklearn}
(1/2-1/k)\cdot T + R(T).
\end{equation}
To prove this, it suffices to show that for every sequence of examples $(x_1,y_1)\ldots (x_T,y_T)$ which is consistent with $\cH_k$
there exists $h\in \cH$ which makes at most $(1/2-1/k)\cdot T$ mistakes on it.
Indeed, let $h_1\ldots h_k$ such that $y_t = \maj_k(h_1(x_t)\ldots h_k(x_t))$ for $t\leq T$.
Thus, on every example $(x_t,y_t)$ at most $1/2-1/k$ fraction of the $h_i$'s make a mistake on it.
By averaging, this implies that one of the $h_i$ makes at most $(1/2-1/k)T$ mistakes in total,
and \Cref{eq:weaklearn} follows.

Now, by applying online boosting with $W$ as a weak learner, we obtain an algorithm with a mistake-bound of at most 
\[O\Bigl(\frac{R(T)}{1/k}\Bigr) = O\Bigl(k\sqrt{\ldim{\cH}T\log T}\Bigr).\]
Thus, since the Littlestone dimension characterizes the optimal mistake-bound,
letting $D=\ldim{\cH_k}$, we get that
\[(\forall T\geq D) : D\leq O\Bigl(k\sqrt{\ldim{\cH}T\log T}\Bigr),\]
and in particular $D\leq O\Bigl(k\sqrt{\ldim{\cH}D\log D}\Bigr)$,
which implies that 
\begin{align*}
D		      &\leq \tilde O(k^2\ldim{\cH})\\
		      &\leq \tilde Ok^2d + k^2\log k) \tag{\Cref{prop:largeunions}}\\
		      &=\tilde O (k^2 d),
\end{align*}
and finishes the proof of \Cref{eq:boostmaj}.

We next set out to prove \Cref{eq:formula}. The idea is to represent an arbitrary $G$
using a formula which only uses majority-votes and negations.
Let $G:\{0,1\}^k\to\{0,1\}$ be an arbitrary boolean function.
It is a basic fact that $G$ can be represented by a Disjunctive Normal Form (DNF) as follows:
\[G=\bigvee_{i=1}^m \Bigl( \bigwedge_{j=1}^k z_{i,j}\Bigr),\]
where each $z_{i,j} \in \{x_j, \lnot x_j\}$, and $m\leq 2^k$.
Now, note that
\[ \bigwedge_{j=1}^k z_{i,j} = \maj_{2k-1}(z_{i,1},\ldots, z_{i,k}, \underbrace{{\bf 0}, \ldots, {\bf 0}}_{k-1}),\]
and similarly
\[ \bigvee_{i=1}^m \Bigl(\bigwedge_{j=1}^k z_{i,j} \Bigr) = \maj_{2m-1}\Bigl(\bigwedge_{j=1}^k z_{1,j},\ldots, \bigwedge_{j=1}^k z_{m,j}, \underbrace{{\bf 1}, \ldots, {\bf 1}}_{m-1}\Bigr).\]
Thus, $G(\cH_1,\ldots, \cH_k)$ can be written as $\maj_{2m-1}(\cH_1',\ldots \cH_{2m-1}')$,
where for $i>m$, $\cH_i' = \{h_0\}$ is the class which contains the all-zero function $h_0$,
and for $i\leq m$, 
\[\cH_i'=\maj_{2k-1}(\cH''_{i,1}, \ldots, \cH''_{i,2k-1})\]
such that each class $\cH''_{i,j}$ is either $\cH_t$ or its negation $\neg \cH_t$ for some $t\leq k$,
or $\cH''_{i,j}$ is the class~$\{h_1\}$ which contains the all-one function.
We now apply \Cref{eq:boostmaj} to conclude that $\ldim{\cH'_{i}} = \tilde O(k^2d)$ for all $i\leq m$,
and that
\[\ldim{G(\cH_1\ldots \cH_k)} = \tilde O\bigl(m^2(k^2d)\bigr) = \tilde O(2^{2k} k^2d)\]
as required.
\qed

\section{Private Agnostic Learning and Closure of Private Learning }
\label{sec:relabel}

In this section we describe our private learning algorithm. We start by discussing a relabeling procedure (discussed in \ref{sec:main}), explaining the difficulties in designing the procedure and how we overcome them. We then provide a formal description of the relabeling procedure in $\AlgRelabel$ and prove that it can be used to construct a private algorithm that produces hypothesis that has good generalization properties; this is done by presenting an algorithm $\AlgRelabelAndLearn$.

Let $\cH$ be a hypothesis class, and suppose that we have a differentially private learning algorithm $\AAA$ for $\cH$ for the realizable setting. That is, $\AAA$ is guaranteed to succeed in its learning task whenever it is given a labeled database that is consistent with some hypothesis in $\cH$. Now suppose that we are given a database $S$ sampled from some distribution $\DDD$ on $X$ and labeled by some concept $c^*$ (not necessarily in $\cH$). So, $S$ might {\em not} be consistent with any hypothesis in $\cH$, and we cannot directly apply $\AAA$ on $S$. Heuristically, one might first {\em relabel} the database $S$ using some function from $\cH$, and then apply $\AAA$ on the relabeled database. Can we argue that such a paradigm would satisfy differential privacy, or is it the case that the relabeling process ``vaporises'' the privacy guarantees of algorithm $\AAA$?

Building on a result of \cite{BeimelNS15}, we show that it {\em is} possible to relabel the database before applying algorithm $\AAA$ while maintaining differential privacy. As we mentioned in the introduction, the relabeling procedure of 
\cite{BeimelNS15} instantiates the exponential mechanism in order to choose a hypothesis $h$ that is (almost) as close as possible to the original labels in $S$, uses this hypothesis to relabel the database, and applies the given differentially private algorithm $\AAA$ on the relabeled database to obtain an outcome $f$.

Now we want to argue that $f$ has low generalization error. 
We known (by the guarantees of the exponential mechanism) that the hypothesis $h$ with which we relabeled $S$ has a relatively small empirical error on $S$ (close to the lowest possible error). Via standard VC arguments, we also know that $h$ has a relatively small generalization error. Therefore, in order to show that the returned hypothesis $f$ has low generalization error, it suffices to show that $\error_{\DDD}(f,h)$ is small.
This might seem trivial at first sight: Since as $\AAA$ is a PAC learner, and since it is applied on a database $S$ labeled by the hypothesis $h\in\cH$,  it must (w.h.p.)\ return a hypothesis $f$ with small error w.r.t.\ $h$. Is that really the case?

The difficulty with formalizing this argument is that $\AAA$ is only guaranteed to succeed in identifying a good hypothesis when it is applied on an {\em i.i.d.}\ sample from some underlying distribution. This is {\em not} true in our case. Specifically, we first sampled the database $S$ from the underlying distribution, then {\em based on $S$}, we identified the hypothesis $h$ and relabeled $S$ using $h$. For all we know, $\AAA$ might completely fail when executed on such a database (not sampled in an i.i.d.\ manner).\footnote{To illustrate this issue, suppose that the learner $\AAA$ first checks to see if {\em exactly} half of the elements in its input sample are labeled by 1 and {\em exactly} half of them are labeled by 0. If that it the case, then $\AAA$ fails. Otherwise, $\AAA$ identifies a hypothesis $f$ with small empirical error. On an a correctly sampled database (sampled i.i.d.\ from some underlying distribution) the probability that exactly half of the elements will be labeled as 0 is low enough such that $\AAA$ remains a valid learning algorithm. However, if we  first sample the elements, and then choose a hypothesis that evaluates to 1 on exactely half of them, then this breaks the utility guarantees of $\AAA$ completely.}
%
Therefore, before applying $\AAA$ on the relabeled database, we subsample i.i.d.\ elements from it, and apply $\AAA$ on this newly sampled database. Now we know that $\AAA$ is applied on an i.i.d.\ sampled database, and so, by the utility guarantees of $\AAA$, the hypotheses $f$ and $h$ are close w.r.t.\ the underlying distribution. 
However, this subsampling step changes the distribution from which the inputs of $\AAA$ are coming from. This distribution is no longer $\DDD$ (the original distribution from which $S$ was sampled), rather it is the uniform distribution on the empirical sample $S$. This means that what we get from the utility guarantees of $\AAA$ is that $\error_S(f,h)$ is small. We need to show that $\error_{\DDD}(f,h)$ is small. 

If $\AAA$ is a {\em proper} learner, then $f$ is itself in $\cH$, and hence, using standard VC arguments, the fact that $\error_{S}(f,h)$ is small implies that $\error_{\DDD}(f,h)$ is small. However, if $\AAA$ is an improper learner, then this argument breaks because $f$ might come from a different hypothesis class with a much larger VC dimension.

To overcome this difficulty, we will instead relate $\error_S(f,h)$ and $\error_{\DDD}(f,h)$ using the generalization properties of differential privacy. These generalization properties state that if a predicate $t$ was identified using a differentially private algorithm, then (w.h.p.)\ the empirical average of this predicate and its expectation over the underlying distribution are close. More formally, we would like to consider the predicate  $(h{\oplus}f)(x)=h(x)\oplus f(x)$, which would complete our mission because the empirical average of that predicate on $S$ is $\error_S(f,h)$, and its expectation over $\DDD$ is $\error_{\DDD}(f,h)$.
However, while $f$ is indeed the outcome of a differentially private computation, $h$ {\em is not}, and we cannot directly apply the generalization properties of differential privacy to this predicate. Specifically, our relabeling procedure does not reveal the chosen hypothesis $h$. 

We overcome this issue by introducing the following conceptual modification to the relabeling procedure. Let us think about the input database $S$ as {\em two} databases $S=D{\circ}T$. In the relabeling procedure we still relabel all of $S$ using $h$. We show that (a small variant of) this relabeling procedure still satisfies differential privacy w.r.t.\ $D$ {\em even if the algorithm publicly releases the relabeled database $T$}. 
This works in our favour because given the relabeled database $T$ we can identify a hypothesis $h'\in\cH$ that agrees with it, and by standard VC arguments we know that $\error_{\DDD}(h,h')$ is small (since both $h,h'$ come from $\cH$). In addition, $h'$ is computed by post-processing the relabeled database $T$ which we can view as the result of a private computation w.r.t.\ $D$. Therefore, we can now use the generalization properties of differential privacy to argue that
$\error_D(f,h)\approx\error_{\DDD}(f,h)$, which would allow us to complete the proof.
%
We remark that the conceptual modification of treating $S$ as two databases $S=D{\circ}T$ is crucial for our analysis. We do not know if the original relabeling procedure of \cite{BeimelNS15} can be applied when $\AAA$ is an improper learner. 

In \cref{alg:LabelBoostProcedure} we formally describe $\AlgRelabel$. We next provide an informal description of the algorithm.  Let $\cH$ be a hypothesis class, and let $q$ be a score function. Algorithm $\AlgRelabel$ takes two input databases $D,T\in(X\times\{0,1\})^*$, where the labels in $D$ and $T$ are arbitrary. The algorithm {\em relabels} $D$ and $T$ using a hypothesis $h\in\cH$ with near optimal score $q(D,h)$. The output of this algorithm is the two relabeled databases $\tilde{D}$ and $\tilde{T}$. Observe that algorithm $\AlgRelabel$ is clearly {\em not} differentially private, since it outputs its input database (with different labels). Before formally presenting algorithm $\AlgRelabel$, we introduce the following definition.

\begin{definition}
\label{def:matched-sensitivity}
Let $X$ be a domain and let $\cH$ be a class of functions over $X$. A function $q:(X \times \ourset{0,1})^*\times \cH\rightarrow\R$ has {\em matched-sensitivity} $k$ if for every $S\in (X \times \ourset{0,1})^*$, every $(x,y),(x'y')\in X \times \ourset{0,1}$, and every $h,h'\in \cH$ that agree on every element of $S$ we have that
$$
\left| q\left(S\cup\{(x,y)\},h\right) - q\left(S\cup\{(x',y')\},h'\right) \right|\leq k.
$$
\end{definition}

In words, a score function $q$ has low {\em matched-sensitivity} if 
given ``similar'' 
databases it assigns ``similar'' scores 
to ``similar'' solutions. 
Note that if a function $q$ has matched-sensitivity 1, then in particular, it has (standard) sensitivity (at most) 1.

\begin{example}
Let $\cH$ be a concept class over $X$. 
Then, the score function $q(S,h)$ that takes a labeled database $S\in(X\times\{0,1\})^*$ and a concept $h\in H$ and returns the number of errors $h$ makes on $S$ has matched-sensitivity at most 1.
\end{example}

\begin{algorithm}
\caption{$\AlgRelabel$}\label{alg:LabelBoostProcedure} 
\vspace{2pt}
{\bf Global parameters:} 
\begin{itemize}
\item 
$\cH$ is a concept class over a domain $X$, 
\item
$q:(X\times\{0,1\})^*\times \cH\rightarrow\R$ is a score function with matched-sensitivity at most 1 (see Definition~\ref{def:matched-sensitivity}), which given a labeled database assigns scores to concepts from $\cH$,
\end{itemize}
{\bf Inputs:} Labeled databases $D,T\in(X\times\{0,1\})^*$. We denote $S=D{\circ}T$.
\begin{enumerate}[rightmargin=10pt,itemsep=1pt,topsep=3pt]

\item Let $P=\{p_1,\ldots,p_\ell\}$ be the set of all unlabeled points appearing at least once in $S$.

\item 
 Let $H=\Pi_{\cH}(P)=\ourset{h|_{P} :h\in \cH}$, where $h\vert_{P}$ denotes the restriction of $h$ to $P$
		(i.e., $H$ contains all patterns of $\cH$ when restricted to $P$).

\item \label{step:Oneexpmech} Choose $h\in H$ using the exponential mechanism with privacy parameter $\eps{=}1$, score function $q$, solution set $H$, and the database $D$.

\item \label{step:Onerelabel} Relabel $S$ using $h$, and denote the relabeled databases as $S^h=D^h{\circ}T^h$. That is, if $D=(x_i,y_i)_{i=1}^d$ then   $D^h=(x_i,h(x_i))_{i=1}^d$,  and similarly with $T^h$.

\item \label{step:outputRelabeled} 

Output $D^h,T^h$.

\end{enumerate}
\end{algorithm}

We next present an algorithm $\AlgRelabelAndLearn$ and analyze its properties. 
This algorithm is an abstraction of parts of $\AlgPrivateAgnostic$ and
$\AlgClosureLearn$ and is used for unifying the proofs of privacy and correctness of these algorithms. 
We start with an informal description of  algorithm $\AlgRelabelAndLearn$.
The algorithm first applies the relabeling algorithm $\AlgRelabel$ and then applies a private algorithm to the relabeled database. For the analysis of our algorithms in the sequence,
$\AlgRelabelAndLearn$ also publishes part of the relabeled database. We prove that $\AlgRelabelAndLearn$ guarantees differential privacy w.r.t.\ to the part of the database that it did not publish.

\begin{algorithm}
\caption{$\AlgRelabelAndLearn$}\label{alg:LabelAndLearn} 
\vspace{2pt}
{\bf Global parameters:} 
\begin{itemize}
\item 
$\cH$ is a concept class over a domain $X$, 
\item
$q:(X\times\{0,1\})^*\times \cH\rightarrow\R$ is a score function with matched-sensitivity at most 1, which given a labeled database assigns scores to concepts from $\cH$,
\item
$\Alg$ is a an $(\epsilon,\delta)$-differentially private algorithm.
\end{itemize}
{\bf Inputs:} Labeled databases $D,V,W\in(X\times\{0,1\})^*$. We denote $S=D{\circ}V{\circ}W$.
\begin{enumerate}[rightmargin=10pt,itemsep=1pt,topsep=3pt]

\item Execute $\AlgRelabel(D,V{\circ}W )$ with score function $q$ and hypothesis class $\cH$
to obtain relabeled databases $\tilde{D},\tilde{V},\tilde{W}$.

\item
Let $\overline{h}$ be a hypothesis in $\cH$ that is consistent with $\tilde{V}$.

\item \label{step:output} 

Output $\Alg(S),\tilde{V},\overline{h}$.

\end{enumerate}
\end{algorithm}

In Lemma~\ref{lemma:RelabelNewPrivacy}, we  analyze the privacy properties of algorithm $\AlgRelabelAndLearn$.

\begin{lemma}\label{lemma:RelabelNewPrivacy}
Let $\Alg$ be an $(\epsilon,\delta)$-differentially private algorithm
and $q$ be a score function with matched-sensitivity 1. 
Then, for every $V$, algorithm $$\AlgRelabelAndLearn^V(D,W)=\AlgRelabelAndLearn(D,V,W)
$$
is $(\epsilon+3,4e\delta)$-differentially private w.r.t.\ $D{\circ}W$.
In particular, $\Alg(\AlgRelabel(D,T))$ is  $(\epsilon+3,4e\delta)$-differentially private.
\end{lemma}


\begin{proof}
Fix a database $V$, and let $D_1{\circ}W_1$ and $D_2{\circ}W_2$ be two neighboring databases. We assume 
that $D_1{\circ}W_1$ and $D_2{\circ}W_2$ differ on their $D$ portion, so that $W_1=W_2=W$ and $D_1=D\cup\{(\ppp_1,y_1)\}$ and $D_2=D\cup\{(\ppp_2,y_2)\}$. The analysis for the other case is essentially identical. 
Consider the executions of $\AlgRelabel$ on 
$S_1=D_1{\circ}V{\circ}W$ and on $S_2=D_2{\circ}V{\circ}W$, and
denote by $H_1,P_1$ and by $H_2,P_2$ the elements $H,P$ as they are in the executions of algorithm $\AlgRelabel$ on $S_1$ and on $S_2$.

Since $S_1$ and $S_2$ are neighbors, it follows that  $|P_1\setminus P_2|\leq 1$ and  $|P_2\setminus P_1|\leq 1$. 
Let $K=P_1\cap P_2$. 
Since every pattern in $\Pi_{\cH}(K)$ has at most two extensions in $\Pi_{\cH}(H_t)$, we get that for every $t \in \ourset{1,2}$.
$$|\Pi_\cH(K)|\leq|\Pi_\cH(P_t)|\leq2|\Pi_\cH(K)|.$$
Thus, $|H_1|\leq2|H_2|$ and similarly $|H_2|\leq2|H_1|$.

More specifically, for every $t\in \ourset{1,2}$ and every pattern $h\in\Pi_C(K)$ there are either one or two (but not more) patterns in $H_t$ that agree with $h$ on $K$.
We denote these one or two patterns by $h_t^{(0)}$ and $h_t^{(1)}$, which may be identical if only one unique pattern exists.
By the fact that $q$ has matched-sensitivity at most 1, for every $t_1,t_2 \in \ourset{1,2}$ and every $b_1,b_2\in\ourset{0,1}$ we have that 
\begin{align*}
|q(D_1,h^{(b_1)}_{t_1})-q(D_2,h^{(b_1)}_{t_2})| = |q(D\cup\{(\ppp_1,y_1)\},h_1)-q(D\cup\{(\ppp_2,y_2)\},h_2)|
\leq 1,
\end{align*}
where the last inequality is because $h^{(b_1)}_{t_1}$ and $h^{(b_2)}_{t_2}$ agree on every point in $D$ and because $q$ has matched-sensitivity at most 1.

For every $h \in \Pi_\cH(K)$ and $t \in \ourset{1,2}$, let $w_{t,h}$ be the probability that the exponential mechanism chooses either 
$h^{(0)}_{t}$ or $h^{(1)}_{t}$ in \stepref{step:Oneexpmech} of the execution of $\AlgRelabel$ on $S_i$. We get that for every $h\in\Pi_C(K)$,
\begin{eqnarray*}
w_{1,h}
&\leq& \frac{\exp(\frac{1}{2}\cdot q(D_1,h^{(0)}_{1}))+\exp(\frac{1}{2}\cdot q(D_1,h^{(1)}_{1}))}{\sum_{f\in \Pi_{\cH}(P_1)}{\exp(\frac{1}{2}\cdot q(D_1,f))}}\\
&\leq& \frac{\exp(\frac{1}{2}\cdot q(D_1,h^{(0)}_{1}))+\exp(\frac{1}{2}\cdot q(D_1,h^{(1)}_{1}))}{\sum_{f\in\Pi_\cH(K)}{\exp(\frac{1}{2}\cdot q(D_1,f^{(0)}_{1}))}}\\
&\leq& \frac{\exp(\frac{1}{2}\cdot [q(D_2,h^{(0)}_{2})+1])+\exp(\frac{1}{2}\cdot [q(D_2,h^{(1)}_{2})+1])}{\frac{1}{2}\sum\limits_{f\in\Pi_\cH(K)}\left(\exp(\frac{1}{2}[q(D_2,h^{(0)}_{2})-1])+
\exp(\frac{1}{2}[q(D_2,h^{(1)}_{2})-1])\right)}\\
&\leq& 2 e\cdot \frac{\exp(\frac{1}{2}\cdot q(D_2,h^{(0)}_{2}))+\exp(\frac{1}{2}\cdot q(D_2,h^{(1)}_{2}))}{\sum_{f\in \Pi_{\cH}(P_2)}{\exp(\frac{1}{2}\cdot q(D_2,f))}}\\
&\leq& 4e \cdot
w_{2,h}.
\end{eqnarray*}


We are now ready to conclude the proof. 
For every $h\in\Pi_\cH(K)$, let 
$I_t$ be the event that the exponential mechanism chooses in \stepref{step:Oneexpmech} of the execution on $S_t$ either
$h^{(0)}_t$ or $h^{(1)}_t$ and
$h_{t}$ be the random variable denoting the pattern that the exponential mechanism chooses in \stepref{step:Oneexpmech} of the execution on $S_t$ conditioned on the event $I_t$. 
Observe that $S^{h_0}$ and $S^{h_1}$ are distributions on neighboring databases;
thus, applying the differentially private $\Alg$ on them satisfies differential privacy,
i.e., for every possible sets of outputs $F$ of $\Alg$:
\begin{equation*}
\Pr\left[\Alg\left( S_1^{h_1} \right) \in F\right]\\
\leq e^{\epsilon}\Pr\left[\Alg\left(  S_2^{h_2} \right) \in F \right]+\delta.
\end{equation*}


Recall that  algorithm $\AlgRelabelAndLearn$ returns {\em three} outcomes: the relabeled database $V^h$,  hypothesis $\overline{h}$ that is consistent with $V^h$, and the output of algorithm $\Alg$. 
As $\overline{h}$ is computed from $V^h$, we can consider it as post-processing and ignore it, and assume for the the privacy analysis that $\AlgRelabelAndLearn$ only has two outputs: $V^h$ and the output of algorithm $\Alg$. 
Also recall that the database $V$ is fixed, and observe that once the hypothesis $h$ is fixed (in \stepref{step:Oneexpmech} of algorithm $\AlgRelabel$), the relabeled database $V^h$ is also fixed.
Furthermore, for every $h\in\Pi_\cH(K)$ we have that $V^{h^{(0)}_{t}}=V^{h^{(1)}_{t}}$, since $h^{(0)}_{t}$ and $h^{(1)}_{t}$ agree on all of $V$. 

Let $F\subseteq(X\times\{0,1\})^*\times R$ be a set of possible outcomes for algorithm $\AlgRelabelAndLearn$, where $R$ is the range of algorithm $\Alg$. 
For every $h$ we denote 
$$F_{h}=\left\{ r\in R : (V^h,r)\in F \right\}.$$
Observe that for every $h\in\Pi_C(K)$ we have that
$$
F_{h^{(0)}_{1}}=F_{h^{(1)}_{1}}=F_{h^{(1)}_{2}}=F_{h^{(2)}_{2}}=F_h,
$$
because $h^{(0)}_{1},h^{(1)}_{1},h^{(0)}_{2},h^{(1)}_{2}$ agree on all points in $V$. We calculate,
\begin{align*}
\Pr[\AlgRelabelAndLearn\left( S_1 \right)\in F]&= \sum_{h\in\Pi_\cH(K)} w_{1,h} 
\cdot\Pr\left[\AlgRelabelAndLearn\left( S_1 \right)\in F \Big| 
I_t
\right]\\
&=\sum_{h\in\Pi_\cH(K)} w_{1,h} 
\cdot
\Pr\left[\Alg\left( S_1^{h_1} \right) \in F_{h_{1}}\right]\\
&\leq \sum_{h\in\Pi_\cH(K)}4e \; w_{2,h}
\cdot
\left(e^{\epsilon}\Pr\left[\Alg\left(  S_2^{h_2} \right) \in F_{h_{2}} \right]+\delta\right)
\\
&\leq  e^{\epsilon+3}\cdot\Pr[\BBB\left( S_2 \right)\in F]+4e\delta.
\end{align*}
\end{proof}

The next claim proves that $\AlgRelabel$ returns a hypothesis whose score  is close to the hypothesis with smallest score in the class $\cH$.  

\begin{claim}\label{claim:RelabelUtility}
Fix $\alpha$ and $\beta$, and let $S=D{\circ}T\in(X\times\{0,1\})^*$ be a labeled database such that $$|D|\geq\frac{2}{\alpha}\ln\left(\frac{1}{\beta}\right)+\frac{2\VC(\cH)}{\alpha}\ln\left(\frac{e|S|}{\VC(\cH)}\right).$$
Consider the execution of $\AlgRelabel$ on $S$, and let $h$ denote the hypothesis chosen on \stepref{step:Oneexpmech}. With probability at least $(1-\beta)$ we have that $q(D,h)\leq\min_{c\in \cH}\{q(D,c)\}+\alpha|D|$.
In particular, assuming that $|D|\geq|S|/2$, it suffices that
$$
|D|\geq\frac{4}{\alpha}\ln\left(\frac{1}{\beta}\right)+\frac{10\VC(\cH)}{\alpha}\ln\left(\frac{20e}{\alpha}\right).
$$
\end{claim}

\begin{proof}
Note that by Sauer-Shelah-Perles lemma, 
\begin{eqnarray*}
|H|&=&|\Pi_\cH(P)| \leq \left(\frac{e|P|}{\VC(\cH)}\right)^{\VC(\cH)}
\leq\left(\frac{e|S|}{\VC(\cH)}\right)^{\VC(\cH)}.
\end{eqnarray*}

As $H$ contains all patterns of $\cH$ restricted to $S$, the set $H$ contains a pattern $f^*$ s.t.\ $q(D,f^*)=\min_{c\in \cH}\ourset{q(D,c)}$. 
Hence, Proposition~\ref{prop:ExpMech} (properties of the exponential mechanism) ensures that the probability of the exponential mechanism choosing an $h$ s.t.\ $q(D,h)>\min_{c\in \cH}\{q(D,c)\}+\alpha$ is at most
$$
|H|\cdot\exp(-\frac{\alpha |D|}{2})\leq
\left(\frac{e|S|}{\VC( \cH)}\right)^{\VC( \cH)}\cdot\exp(-\frac{\alpha |D|}{2}),
$$
which is at most $\beta$ whenever $|D|\geq\frac{2}{\alpha}\ln(\frac{1}{\beta})+\frac{2\VC( \cH)}{\alpha}\ln\left(\frac{e|S|}{\VC( \cH)}\right)$.
\end{proof}

Let $f$ denote the hypothesis returned by $\AAA$ and let $h$ be a hypothesis consistent with the pattern chosen on \stepref{step:Oneexpmech} of $\AlgRelabel$. 
The next lemma relates the generalization error $\error_{\DDD}(f,h)$ to the empirical error $\error_{D}(f,h)$.

\begin{lemma}\label{lemma:RelabelImportantUtility}
Fix $\alpha$ and $\beta$, and let $\mu$ be a distribution on $X\times\ourset{0,1}$ and $\DDD$ be  the marginal distribution on unlabeled examples from $X$. Furthermore,
let $S=D{\circ}V{\circ}W\in(X\times\{0,1\})^*$ be database sampled i.i.d.\ from $\mu$ such that $$|V|\geq O\left( \frac{\VC( \cH)\ln\left(\frac{1}{\alpha}\right)+\ln\left(\frac{1}{\beta}\right)}{\alpha} \right),$$
and
$$|D| \geq O\left( \frac{\VC( \cH)+\ln\left(\frac{1}{\beta}\right)}{\alpha^2} \right).$$
Consider the execution of $\AlgRelabelAndLearn$ on $S$,  let $h \in \cH$ be a hypothesis consistent with the pattern chosen on \stepref{step:Oneexpmech} of $\AlgRelabel$ and assume that $\Alg$ outputs some hypothesis $f$. With probability at least $1-O(\beta+\delta|D|)$ we have that 
$$
\error_\DDD(f,h) \leq O(\error_D(f,h) + \alpha)
.
$$
\end{lemma}

\begin{proof}
Let $\overline{h}$ be the third output of $\AlgRelabelAndLearn$, i.e., a hypothesis from $\cH$ that is consistent with $V^h$.
Since $\overline{h}$ and $h$ agree on $V$
and $|V|$ is big enough, 
by \Cref{thm:VCconsistant}, with probability at least $1-\beta$ (over sampling $V$), 
\begin{align}
\label{eq:RelabelDDDoverlinehh}
\error_{\DDD}(\overline{h},h)\leq\alpha.
\end{align}
Since $|D|$ is big enough, by  \Cref{thm:VCagnostic} (applied to $\cH \oplus \cH$ and the distribution $\mu$ that samples $x$ according to $\DDD$ and labels it with $0$), with probability at least $1-\beta$,
\begin{align}
\label{eq:RelabelDoverlinehh}
\error_{D}(\overline{h},h)\leq \error_{\DDD}(\overline{h},h) +\alpha \leq 2 \alpha.
\end{align}

 We will now use the generalization properties of differential privacy to argue that $\error_\DDD(f,h)$ is small. 
By Lemma~\ref{lemma:RelabelNewPrivacy}, 
algorithm $\AlgRelabelAndLearn$ is $(O(1),O(\delta))$-differentially private w.r.t.\ the database $D$. In addition, by post-processing the outcomes of $\AlgRelabelAndLearn$ (the hypotheses $f$ and  $\overline{h}$) we can define the following predicate $\test:X\times\{0,1\}\rightarrow\{0,1\}$ where $\test(x,y)=1$ if $\overline{h}(x)\neq f(x)$, and $\test(x,y)=0$ otherwise.
Now observe that
\begin{align}
\label{eq:RelabelDDDerrortest}
    \error_{\DDD}(f,\overline{h})&=
    \Pr_{x\sim\DDD}[\overline{h}(x)\neq f(x)]
    =\E_{x\sim\DDD}[\1\{\overline{h}(x)\neq f(x)\}]
    =\E_{(x,y)\sim\mu}[\test(x,y)].
\end{align}
Similarly,
\begin{align}
\label{eq:RelabelDerrortest}
\error_D\left(f,\overline{h}\right) = \frac{1}{|D|}\sum_{(x,y)\in D}\1\left\{\overline{h}(x)\neq f(x)\right\} 
=
\frac{1}{|D|}\sum_{(x,y)\in D}\test(x,y).
\end{align}

Recall that $\test$ is the result of a private computation on the database $D$ (obtained as a post-processing of the outcomes of $\AlgRelabelAndLearn$).
Also observe that since $\AlgRelabelAndLearn$ is $(O(1),O(\delta))$-differentially private, it is in particular, $\left(O(1),O\left(\delta+\frac{\beta}{|D|}\right)\right)$-differentially private for every choice of $\beta$ and $|D|$.
Hence, assuming $|D|\geq O\left(\frac{1}{\alpha}\log\frac{1}{\beta}\right)$, \Cref{thm:GeneralizationMultiplicative} (the generalization properties of differential privacy) states that with probability at least $1-O(\delta|D|+\beta)$, 
\begin{align}
    \label{eq:RelabelDPgeneralization}
    \E_{(x,y)\sim\mu}[\test(x,y)] & \leq O\left(\frac{1}{|D|}\sum_{(x,y)\in D}\test(x,y)+
    \frac{1}{|D|} \log\left(\frac{1}{\delta|D|+\beta}\right)\right)
    \nonumber \\
    & \leq O\left(\frac{1}{|D|}\sum_{(x,y)\in D}\test(x,y)+
    \frac{1}{|D|} \log\left(\frac{1}{\beta}\right)\right)
    \nonumber \\
    & \leq O\left(\frac{1}{|D|}\sum_{(x,y)\in D}\test(x,y)+\alpha\right).
\end{align}
So, by \cref{eq:RelabelDDDerrortest}, \Cref{eq:RelabelDerrortest}, and \Cref{eq:RelabelDPgeneralization}, with probability at least $1-O(\beta+\delta|D|)$
\begin{align}
\label{eq:RelabelDDDfoverlineh}
\error_{\DDD}(f,\overline{h})& \leq O( \error_{D}(f,\overline{h}) +\alpha).
\end{align}
Thus, the next inequality, which concludes the proof, holds with probability $1-O(\beta+\delta|D|)$.
\begin{align*}
    \error_{\DDD}(f,h) & = \error_{\DDD}(f,\overline{h})
    + \error_{\DDD}(\overline{h},h) 
    \\
  & \leq O(\error_{D}(f,\overline{h}) +\alpha) 
  & \text{(by \Cref{eq:RelabelDDDoverlinehh,eq:RelabelDDDfoverlineh})}\\
  & \leq O(\error_{D}(f,h) + \error_{D}(h,\overline{h})+\alpha) \\
  & \leq O(\error_{D}(f,h) +\alpha) 
  & \text{(by \Cref{eq:RelabelDoverlinehh}).} 
\end{align*}
\end{proof}

\section{Private PAC Implies Private Agnostic PAC}
\label{sec:agnostic}
In this section we show that private learning implies private agnostic learning (with essentially the same sample complexity) even for improper learning algorithms. Algorithm $\AlgPrivateAgnostic$, the agnostic algorithm for a class $\cH$, first 
applies algorithm $\AlgRelabel$ on the data and relabels the sample using a hypothesis in $\cH$ that has close to minimal empirical error, and then uses the  private learning  algorithm (after sub-sampling) to learn the relabeled database.

\begin{algorithm}
\caption{$\AlgPrivateAgnostic$}
\vspace{2pt}
{\bf Inputs:} A labeled sample $S\in(X\times\{0,1\})^m$.\\[2pt]
{\bf Auxiliary algorithm:} A private learner $\AAA$ for the concept class $ \cH$.

\begin{enumerate}[rightmargin=10pt,itemsep=1pt,topsep=3pt]

\item Partition $S$ into $S=D{\circ}T$, where $|D|=|T|=|S|/2$.

\item 
\label{step:agnosticRelabel}
Execute $\AlgRelabel$ with input $D,T$ and score function $q(D,h)=|D|\cdot\error_D(h)$ to obtain relabeled databases $\tilde{D},\tilde{T}$.

\item Execute a private  empirical learner on $\tilde{D}$:
\label{step:Agnostc-empirical}
Choose $|D|/9$ samples with replacements from $\tilde{D}$. Denote the resulting database by $Q$
and let $f\leftarrow\AAA(Q)$.
\item Return $f$.
\end{enumerate}
\end{algorithm}

\begin{theorem}[\Cref{thm:agnostic} Restated]
Let $0<\alpha,\beta,\delta<1$,  $m\in\N$, and $\Alg$ 
be a  
$(1,\delta)$-differentially private $(\alpha,\beta)$-accurate PAC learner for $ \cH$
with sample complexity $m$. Then, $\AlgPrivateAgnostic$ is an $(O(1),O(\delta))$-differentially private $(O(\alpha),O(\beta+\delta n))$-accurate {\em agnostic} learner for $ \cH$ with sample complexity
$$n=O\left(m + \frac{1}{\alpha^2}\left(\VC( \cH)+\log\frac{1}{\beta}\right)\right).$$
\end{theorem}

\begin{proof}
The privacy properties of the algorithm are straightforward. Specifically, by Lemma~\ref{lem:empirical}, \stepref{step:Agnostc-empirical} the algorithm (applying $\Alg$ on a subsample from $\tilde{D}$) satisfies $(O(1),O(\delta))$-differential privacy. Algorithm $\AlgPrivateAgnostic$ is, therefore, $(O(1),O(\delta))$-differentially private by Lemma~\ref{lemma:RelabelNewPrivacy}. In particular, if $\Alg$ is $(1,0)$-differentially private then $\AlgPrivateAgnostic$ is $(O(1),0)$-differentially private.

As for the utility analysis, fix a target distribution $\mu$ over $X\times\{0,1\}$, and denote $$\Delta=\min_{c\in  \cH}\{\error_{\mu}(c)\}.$$ Also let $\DDD$ denote the marginal distribution on unlabeled examples from $X$.
 Let $S$ be a sample containing $n$ i.i.d.\ samples from $\mu$, and denote $S=D{\circ}T$ where $|D|=|T|=|S|/2$. By \Cref{thm:VCagnostic}  (the agnostic VC generalization bound), assuming that $|S|\geq O\left( \frac{1}{\alpha^2}\left(\VC( \cH)+\ln\frac{1}{\beta}\right) \right)$, with probability at least $1-\beta$ (over sampling $S$), the following event occur.
 \begin{enumerate}[leftmargin=75pt]
   \item[Event $E_1:\;\;$] $\forall c\in \cH$ we have $\left| \error_{\mu}(c)-\error_{D}(c) \right|\leq\alpha$.
 \end{enumerate}
We continue with the analysis assuming that this event occurs, and show that (w.h.p.)\ the hypothesis $f$ returned by the algorithm has low generalization error. 
Consider the execution of $\AlgPrivateAgnostic$ on $S$.
In \stepref{step:agnosticRelabel} we apply algorithm $\AlgRelabel$ to obtain the relabeled databases $\tilde{D},\tilde{T}$.
Let $h\in \cH$ be a hypothesis extending the pattern used by algorithm $\AlgRelabel$ to relabel these databases. By Claim~\ref{claim:RelabelUtility}, assuming that $|D|$ is big enough, with probability at least $1-\beta$ it holds that
\begin{align}
\label{eq:agnosticDh}
\error_D(h)\leq\min_{c\in \cH}\{\error_D(c)\}+\alpha.
\end{align}
In this case, by Event $E_1$ we have that 
\begin{align}
\label{eq:agnosticmuh}
\error_{\mu}(h)\leq\error_D(h)+\alpha\leq
\min_{c\in \cH}\{\error_D(c)\}+2\alpha
\leq \min_{c\in \cH}\{\error_{\mu}(c)\}+3\alpha
=\Delta+3\alpha.
\end{align}
Recall that $\AAA$ is executed on the database $Q$ containing $|\tilde{D}|/9$ i.i.d.\ samples from $\tilde{D}$. By Lemma~\ref{lem:empirical}, with probability at least $1-\beta$, the hypothesis $f$ chosen in \stepref{step:Agnostc-empirical} satisfies 
\begin{align}
\label{eq:agnosticDfh}
\error_{D}(f,h)=\error_{\tilde{D}}(f)\leq\alpha.
\end{align}
%
%
%
By Lemma~\ref{lemma:RelabelImportantUtility} and \Cref{eq:agnosticDfh}
with probability at least $1-O(\beta+|D|\delta)$
\begin{align}
\label{eq:agnosticDDDfh}
\error_{\DDD}(f,h) \leq O(\error_{D}(f,h)+\alpha) \leq O(\alpha).
\end{align}
Finally, by \Cref{eq:agnosticmuh,eq:agnosticDDDfh}
\begin{align*}
\error_{\mu}(f)&\leq \error_{\DDD}(f,h)+\error_{\mu}(h)
\leq \Delta+O(\alpha).
\end{align*}
\end{proof}

\section{Closure of Private Learning}\label{sec:privacyproof}


In this section we prove \Cref{thm:privacy} -- if $\cH_1,\dots,\cH_k$ are privately learnable, then $G(\cH_1,\dots,\cH_k)$ is privately learnable.

\remove{
\begin{theorem}[A Closure Theorem for Private Learning: Proper Setting]\label{thm:privacyProper}
Let $G:\{0,1\}^k\to\{0,1\}$ be a boolean function.
Let~$\cH_1,\ldots,\cH_k\subseteq\{0,1\}^X$ be classes 
that are $(\epsilon,\delta)$-differentially private and $(\alpha,\beta)$-accurate learnable by a {\em proper} learning algorithms with sample complexity $m_i(\alpha,\beta,\epsilon,\delta)$ respectively. 
Then, $G(\cH_1,\ldots.\cH_k)$ is $(\eps,\delta)$-private and $(\alpha,\beta)$-accurate learnable by a proper learning algorithm with sample complexity
\[ \tilde{O}\left(  \frac{k^3 (\sum_{i=1}^k\VC(\cH_i)+\log(\frac{k}{\beta}))}{\alpha^2 \epsilon}+\frac{\sum_{i=1}^k m_i\Bigl(\frac{\alpha}{4k},\frac{\beta}{4k},1,\delta\Bigr)}{\epsilon}\right).\]
\end{theorem}
}

\begin{theorem}[Closure Theorem for Private Learning]
\label{thm:privacy}
Let $G:\{0,1\}^k\to\{0,1\}$ be a boolean function and $\cH_1,\ldots,\cH_k\subseteq\{0,1\}^X$ be classes 
that are $(\epsilon,\delta)$-differentially private and $(\alpha,\beta)$-accurate learnable by a possibly {\em improper} learning algorithms with sample complexity $m_i(\alpha,\beta,\epsilon,\delta)$ respectively. 
Then, $G(\cH_1,\ldots,\cH_k)$ is $(O(1),O(\delta))$-private and $(O(\alpha),O(\beta+\delta m))$-accurate learnable with sample complexity
\[
m=O \left( \frac{k^3\VC(G(\cH_1,\ldots,\cH_ k))+k^2\ln\left(\frac{k}{\beta}\right)}{\alpha^2}
 +\sum_{i=1}^k m_i\Bigl(\frac{\alpha}{k},\frac{\beta}{k},1,\delta\Bigr)\right).
\]
\end{theorem}

To prove \cref{thm:privacy}, we  present $\AlgClosureLearn$ -- a generic transformation of  private learning algorithms $\Alg_1,$ $\dots,\Alg_k$ for the classes $\cH_1,\dots,\cH_k$ respectively to a private learner 
for $G(\cH_1,\dots,\cH_k)$. This transformation could be applied to proper as well as improper learners, and to a learners that preserves pure or approximate privacy.
Given a labeled sample $S$ of size $N$, algorithm $\AlgClosureLearn$ finds hypotheses 
$h_1,\ldots,h_k$ in steps, where in  the $i$'th step, the algorithm finds a hypothesis $h_i$ such that 
$h_1,\dots,h_i$ have a  completion $c_{i+1},\dots,c_k$ to a hypothesis
$G(h_1,\dots,h_i,c_{i+1}$, $\dots,c_k)$ with small error (assuming that $h_1,\dots,h_{i-1}$ have a good completion).
In the $i$'th step, $\AlgClosureLearn$ relabels the input sample $S$ so that the relabeled sample is realizable by $\cH_i$.
The relabeling $h$ is chosen using $\AlgRelabel$ in a way that guarantees completion to a hypothesis with small empirical error. 
	That is, using an appropriate score-function in $\AlgRelabel$ (i.e., in the exponential mechanism), it is guaranteed that 
	for the hypotheses $h_1,\dots,h_{i-1}$ computed in the previous steps  there are some 
	$c_{i+1}\in\cH_i,\dots,c_k\in\cH_k$  such that the function $G(h_1,\dots,h_{i-1},h,c_{i+1},\dots,c_k)$ has a small loss with respect to the original sample $S$.
	The relabeled sample is fed  (after subsampling) to the private algorithm $\Alg_i$ to produce a hypothesis $h_i$
	and then the algorithm proceeds to the next step $i+1$.

\begin{algorithm}
\caption{$\AlgClosureLearn$} 
\label{alg:ClosureLearn} 
{\bf Input:} A labeled sample $S\in(X\times\{0,1\})^N$,
where $N$ will be fixed later.\\
{\bf Auxiliary Algorithms:} Private learners $\Alg_1,\dots,\Alg_k$  for the class $\cH_1,\dots,\cH_k$ respectively.
\begin{enumerate}[rightmargin=10pt,itemsep=1pt,topsep=0pt]
\item Partition $S$ into $k$ parts
$S=S_1{\circ}S_2{\circ}\cdots{\circ}S_k$ -- the size of the
$S_i$ will be determined later.
\item For every $i\in [k]$: \label{step:Loop}
\begin{enumerate}
    \item \label{step:ClosuefirstSampling}
    Partition $S_i$ into $S_i=D_i \circ T_i$, where $|D_i|=|T_i|=|S_i|/2$. 
    \item 
    \label{step:ClosureRelabel}
    Execute $\AlgRelabel$ with input $D,T$, hypothesis class $\cH_i$, and score function
    \begin{equation}
    \label{eq:q}
    q(S_i,z)=|S_i| \cdot \min_{c_{i+1}\in \cH_{i+1},\ldots, c_{k}\in \cH_{k}}\error_{S_i}(G(h_1,\ldots,h_{i-1}, h,c_{i+1},\ldots,c_k)),
    \end{equation}
    to obtain relabeled databases $\tilde{D_i},\tilde{T_i}$.
    \item Execute a private empirical learner on $\tilde{D_i}$:
    \label{step:Learn}
    \begin{enumerate}
       \item
       \label{step:choosewithrep}
       Choose $|D_i|/9$ samples with replacements from 
       $\tilde{D_i}$. Denote the resulting database by $Q$.
       \item Execute the private learner $\Alg_i$ on the sample  $Q$ with accuracy parameters $(\alpha/k,\beta/k)$ and privacy parameters $(\epsilon=1,\delta)$. Let $h_i$ be its output.
      \end{enumerate}
\end{enumerate}
\item \label{step:AAA} Output $c=G(h_1,\dots,h_k)$.
\end{enumerate}
\end{algorithm}


In Lemma~\ref{lem:ClosureLearnPrivacy}, we analyze the privacy guarantees of $\AlgLearnComposition$. 

\begin{lemma}\label{lem:ClosureLearnPrivacy}
Let $\epsilon <1$ and assume the algorithms $\Alg_1,\dots,\Alg_k$  are $(1,\delta)$-private.
Then,  $\AlgClosureLearn$ is $(\epsilon,O(\delta))$-differentially private.
\end{lemma}
\begin{proof}
Fix $i\in [k]$ and consider the $i$'th step of the algorithm.
By Lemma~\ref{lem:empirical}, 
\stepref{step:Learn} of algorithm $\AlgClosureLearn$ (i.e.,
sub-sampling with replacement and executing a $(1,\delta)$-private algorithm)  
is $(1,\delta)$-differentially private. 
Thus, by Lemma~\ref{lemma:RelabelNewPrivacy},
\steprefrange{step:ClosureRelabel}{step:Learn} of algorithm $\AlgClosureLearn$ are $(O(1),O(\delta))$-differentially private.
Since each step is executed on a disjoint set of examples, $\AlgClosureLearn$ is $(O(1),O(\delta))$-differentially private.
\end{proof}

In the next lemma we prove that $\AlgClosureLearn$ is an accurate learner for the class $G(\cH1,\dots,\cH_k)$. 

\begin{lemma}\label{lem:ClosureLearnUtilityImproper}
Assume that $\Alg_1,\ldots,\Alg_t$ are $(1,\delta)$-differentially private $(\alpha/k,\beta/k)$-accurate (possibly improper)  learning algorithms for $\cH_1,\ldots,\cH_k$ with sample complexity $m_i(\alpha/k,\beta/k,1,\delta)$. If at each iteration $i$
\[|S_i|\geq O \left( \frac{k^2\VC(G(\cH_1,\ldots,\cH_k))+k\ln\left(\frac{k}{\beta}\right)}{\alpha^2}
 +m_i\Bigl(\frac{\alpha}{k},\frac{\beta}{k},1,\delta\Bigr)\right),\]
then with probability at least $1-O(\beta+k\delta|S_i|)$ we have that $\error_{\DDD}(c)\leq O(\alpha)$,
where $c$ is the hypothesis returned by $\AlgClosureLearn$ on $S$.
\end{lemma}

\begin{proof}
Let $h_1,\ldots,h_k$ be the hypotheses that $\AlgClosureLearn$ computes in \stepref{step:Learn}. 
 We prove by induction that for every $i  \in [k]$ with probability 
at least $1-\frac{O(i)\cdot  \beta}{k}+O(i\cdot\delta|S_i|)$ there exist $c_{i+1} \in \cH_{i+1},\ldots,c_{k} \in \cH_{k}$ such that 
\begin{equation}
\label{eq:inductionNonProper}
\error_{\DDD}(G(h_1,\ldots,h_i,c_{i+1},\ldots,c_{k}) \leq \frac{O(i) \cdot \alpha}{k}.
\end{equation}
The induction basis for $i=0$ is implied by the fact that the examples are labeled by some $G(c_1,\ldots,c_{k})$ from $G(\cH_1,\dots,\cH_k)$.
For the induction step, assume  that there are $c_{i} \in \cH_{i},\ldots ,c_{k} \in \cH_{k}$ such that 
\[\error_{\DDD}\Bigl(G(h_1,\ldots,h_{i-1},c_i,c_{i+1},\ldots,c_{k})\Bigr) \leq \frac{O(i-1) \cdot \alpha}{k}.\] 
We need to prove that with probability at least $1-\frac{O(1)\cdot \beta}{k}-O(\delta|S_i|)$ there are $c'_{i+1} \in \cH_{i+1},\ldots,c'_{k} \in \cH_{k}$ such that 
\[\error_{\DDD}(G(h_1,\ldots,h_{i-1},h_i,c'_{i+1},\ldots,c'_{k})) \leq \frac{O(i)\cdot  \alpha}{k}.\]

Recall that each example in $S$, and hence in $S_i$, is chosen i.i.d.\  from the distribution in $\DDD$.
Since 
\begin{align}
\label{eq;ClosureS1}
\vert S_i \vert \geq O \left( \frac{k^2\VC(G(\cH_1,\ldots,\cH_k))+\ln\left(\frac{k}{\beta}\right)}{\alpha^2}\right),
\end{align}
by \Cref{thm:VCagnostic}
applied to $G(\cH_1,\ldots,\cH_k)\oplus G(\cH_1,\ldots,\cH_k)$, with probability at least $1-\frac{\beta}{k}$ (over the sampling of  $S_i$) the following event occurs:
\begin{enumerate}[leftmargin=75pt]
   \item[Event $E_1:\;\;$] $\forall c\in G(\cH_1,\ldots,\cH_k)$ we have  $\left| \error_{\DDD}(c)-\error_{S_i}(c) \right|\leq\frac{\alpha}{k}$.
\end{enumerate}
We continue proving the induction step assuming that $E_1$ occurs.
The proof of the induction step is as follows:

    Since $E_1$ occurs:
     \begin{align}
    \error_{S_i}&(G(h_1,\ldots,h_{i-1},c_i,c_{i+1},\ldots,c_{k})) \nonumber \\
& \leq      \error_{\DDD}(G(h_1,\ldots,h_{i-1},c_i,c_{i+1},\ldots,c_{k})) + \frac{ \alpha}{k} \nonumber \\
& \leq \frac{(O(i-1)+1) \alpha}{k}.\label{1bimp}
    \end{align}
By the definition of $H$, there is $h=h_{\rm opt}\in H$ that agrees with $c_i$ on $S_i$, and therefore 
    \[q(S_i,h_{\rm opt}) \leq |S_i|\frac{(O(i-1)+1) \alpha}{k}.\]
     By Claim~\ref{claim:RelabelUtility}, if
     \begin{align}
     \rvert S_i\rvert \geq O\left(\frac{k}{\alpha}\ln\left(\frac{k}{\beta}\right)+\frac{k\VC(|\cH_i|)}{\alpha}\ln\left(\frac{k}{\alpha}\right)\right),  \label{2aimp} 
     \end{align} 
    then with probability at least $1-\frac{\beta}{k}$, the exponential mechanism returns~$h\in H$ such that 
    \[q(S_i,h) \leq q(S_i,h_{\rm opt}) + |S_i|\frac{ \alpha}{k}  \leq  |S_i|\frac{(O(i-1)+2) \alpha}{k}.\] 
    We assume that the above event occurs, thus,
    the latter implies that 
    there are $c'_{i+1},\dots,c'_k$ such that 
    \begin{align}
    \error_{S_i}(G(h_1,\ldots,h_{i-1},h,c'_{i+1},\ldots,c'_{k})) \leq \frac{(O(i-1)+2) \alpha}{k}.\label{2bimp}
    \end{align}
   	Since $E_1$ occurs, by  \Cref{2bimp},
    \begin{align}
    \label{eq:ClosureDDDGh}
    \error_{\DDD}(G(h_1,\dots,h_{i-1},h,c'_{i+1},\dots,c'_k)) 
		& \leq 
    \frac{\alpha}{k}+ \error_{S_i}(G(h_1,\dots,h_{i-1},h,c'_{i+1},\dots,c'_k))\nonumber \\
    & \leq \frac{(O(i-1)+3)\alpha}{k}
    . 
    \end{align}    
Since
    \begin{align}
		\label{4aImp}
    \lvert D^h_i \rvert \geq 9m_i\Bigl(\frac{\alpha}{k},\frac{\beta}{k},1,\delta\Bigr),
    \end{align}
    Lemma~\ref{lem:empirical} implies that  \stepref{step:Learn} of
    $\AlgClosureLearn$
    is an $(\frac{\alpha}{k},\frac{\beta}{k})$ empirical learner and, therefore, with probability at least $1-\frac{\beta}{k}$ 
\begin{align}
\label{eq:ClosureD_ihh_i}
\error_{D_i}(h,h_i)=\error_{D_i^h}(h_i) \leq \frac{\alpha}{k}.
\end{align}
Again, we assume in the rest of the proof that the above event occurs.
By Lemma~\ref{lemma:RelabelImportantUtility}, 
since 
\begin{align}
\label{eq:ClosureS4}
    |D_i|=\frac{|S_i|}{2}\geq O\left(
    \frac{k^2( \VC(\cH_i) +\ln\left(\frac{k}{\beta}\right))}{\alpha^2} 
    \right)
\end{align}
with probability at least $1-\frac{O(\beta)}{k}-O(\delta|D_i|)$   
\begin{align*}
\error_{\DDD}(h,h_i) & \leq O\left(\error_{D_i}(h,h_i)+\frac{\alpha}{k}\right).
\end{align*}
Thus, by \Cref{eq:ClosureD_ihh_i},
with probability at least $1-\frac{O(\beta)}{k}$ 
\begin{align}
\label{eq:ClosureDDDhh_i}
\error_{\DDD}(h,h_i) & \leq O\left(\frac{(O(i-1)+O(1))\alpha}{k}\right).
\end{align}
The latter, combined with \Cref{eq:ClosureDDDGh}, implies the induction step: with probability at least $1-\frac{O(\beta)}{k}-O(\delta|D_i|)$
    \begin{align*}
    \error_{\DDD}&(G(h_1,\dots,h_{i-1},h_i,c'_{i+1},\dots,c'_k))  \nonumber \\
    &
    \leq 
    \error_{\DDD}(G(h_1,\dots,h_{i-1},h,c'_{i+1},\dots,c'_k))
    +\error_{\DDD}(h,h_i) \\
    & \leq \frac{(O(i-1)+O(1))\alpha}{k}=\frac{O(i)\cdot\alpha}{k}.
    \end{align*}
By \Cref{eq;ClosureS1,2aimp,4aImp,eq:ClosureS4}, the sample complexity $|S_i|$  the $i$'th step is
 \begin{align*}
  O \left( \frac{k^2\VC(G(\cH_1,\ldots,\cH_k))+\ln\left(\frac{k}{\beta}\right)}{\alpha^2}
 +\frac{k}{\alpha}\ln\left(\frac{k}{\beta}\right)
+ m_i\Bigl(\frac{\alpha}{k},\frac{\beta}{k},1,\delta\Bigr)  + \frac{k^2( \VC(\cH_i) 
  +\ln\left(\frac{k}{\beta}\right))}{\alpha^2} 
  \right) \\
 = 
 O \left( \frac{k^2\VC(G(\cH_1,\ldots,\cH_k))+k\ln\left(\frac{k}{\beta}\right)}{\alpha^2}
 +m_i\Bigl(\frac{\alpha}{k},\frac{\beta}{k},1,\delta\Bigr)\right)
 \end{align*}

To conclude, by a union bound, $\AlgClosureLearn$ returns, with probability at least $1-O(\beta+\delta\sum_{i=1}^k |S_i|)$, a hypothesis 
$G(h_1,\dots,h_k)$ with error less than $O(\alpha)$ with respect to the distribution $\DDD$.
\end{proof}

\paragraph{Proof of \Cref{thm:privacy}.}
\begin{proof}
\Cref{thm:privacy} follows from 
Lemmas~\ref{lem:ClosureLearnPrivacy} and \ref{lem:ClosureLearnUtilityImproper}. Specifically, by Lemma~\ref{lem:ClosureLearnUtilityImproper},
to prove that $\AlgClosureLearn$ is $(O(\alpha),O(\beta+ \delta m))$-accurate it suffices that 
\[\sum_{i=1}^k|S_i|\geq \sum_{i=1}^k O\Biggl(\frac{ k^2\VC(G(\cH_1,\dots,\cH_k)) + k \log(\frac{k}{\beta})}{\alpha^2}+m_i\Bigl(\frac{\alpha}{k},\frac{\beta}{k},1,\delta\Bigr)\Biggl).\]
By Lemma~\ref{lem:ClosureLearnPrivacy}, $\AlgClosureLearn$ is $(O(1),O(\delta))$-differentially private.
\end{proof}

\begin{remark}
Since each $\Alg_i$ is an $(\alpha,\beta)$-accurate learning algorithm for the class $\cH_1$,
$$
m_i\left(\frac{\alpha}{k},\frac{\beta}{k},1,\delta\right)
=
\Omega\left(\frac{k \VC(\cH_i)}{\alpha}\right).
$$
Furthermore, by the Sauer-Shelah-Perles Lemma, 
$
\VC(G(\cH_1,\dots,\cH_k) = \tilde{O}(\sum_{i=1}^k \VC(\cH_i)).
$
Thus, the sample complexity of $\AlgClosureLearn$ is
$$
\tilde{O}\left(\sum_{i=1}^k
m_i\left(\frac{\alpha}{k},\frac{\beta}{k},1,\delta\right)\right)\cdot\poly(k,1/\alpha,\log(1/\beta)).
$$
For constant $k,\alpha,\beta$ this is nearly tight. By using sub-sampling (see e.g., \cite{KasiviswanathanLNRS11, BeimelBKN14}), we can achieve $(\epsilon,O(\delta))$-differential privacy by increasing the sample complexity by a factor of $O(1/\epsilon)$. 
Furthermore, by using private boosting~\cite{DworkRV10}, one can start with a private algorithm that is, for example, $(1/4,\beta)$ accurate and 
get a private algorithm that is $(\alpha,\beta)$ by increasing the sample complexity by a factor of $O(1/\alpha)$, and by simple technique, one can boost $\beta$ by increasing the sample complexity by a factor of $O(\log(1/\beta))$. 
Thus, we get an $(\epsilon,O(\delta))$-differentially private $(\alpha,\beta)$-accurate learner for $G(\cH_1,\dots,\cH_k)$ whose sample complexity is 
$$
\frac{\tilde{O}\left(\sum_{i=1}^k
m_i\left(1/4,1/2,1,\delta\right)\right)}{\epsilon}\cdot\poly(k,1/\alpha,\log(1/\beta)).
$$
\end{remark}

\section*{Acknowledgements}
We thank Adam Klivans and Roi Livni for insightful discussions.

\bibliographystyle{plainnat}
\bibliography{biblio}

\end{document}